\documentclass[onecolumn]{article}

\usepackage{lipsum}
\usepackage{amsfonts}
\usepackage{graphicx}
\usepackage{epstopdf}
\usepackage{microtype}
\usepackage{geometry}
\usepackage{setspace}
\usepackage{algorithm}
\usepackage{algpseudocode}
\usepackage{aliascnt}
\usepackage{amsthm}
\ifpdf
  \DeclareGraphicsExtensions{.eps,.pdf,.png,.jpg}
\else
  \DeclareGraphicsExtensions{.eps}
\fi

\usepackage{hyperref}
\usepackage{xr}
\usepackage{dsfont}
\usepackage[T1]{fontenc}
\usepackage{graphicx}
\usepackage{amsmath,amssymb}
\usepackage{cleveref}
\Crefformat{figure}{#2Fig.~#1#3}
\Crefmultiformat{figure}{Figs.~#2#1#3}{ and~#2#1#3}{, #2#1#3}{ and~#2#1#3}\usepackage{mathtools}
\usepackage{todonotes}
\usepackage{lmodern}
\usepackage{bbm}
\usepackage{tikz}
\usetikzlibrary{spy,decorations.pathreplacing,angles,quotes,calc,arrows,positioning,}
\usepackage{multirow} 
\usepackage{pgfplots}
\pgfplotsset{compat=1.15}
\usepackage{diagbox}
\usepackage[firstinits=true,maxbibnames=99,doi=false,url=false,eprint=true,
backend=biber,style=numeric]{biblatex}

\newcommand{\I}{\mathbbm{1}}
\newcommand{\Id}{\mathrm{Id}}
\newcommand{\R}{\mathbb{R}}
\newcommand{\C}{\mathbb{C}}
\newcommand{\N}{\mathbb{N}}
\newcommand{\Ic}{\mathcal{I}}
\newcommand{\E}{\mathbb{E}}
\newcommand{\Xc}{\mathcal{X}}
\newcommand{\Yc}{\mathcal{Y}}
\newcommand{\Zc}{\mathcal{Z}}

\newcommand{\prob}{\mathrm{p}}
\newcommand{\Prob}{\mathrm{P}}

\newcommand{\expnumber}[2]{{#1}\mathrm{e}{#2}}

\DeclareMathOperator*{\Var}{Var}
\DeclareMathOperator*{\TV}{TV}

\DeclareMathOperator{\prox}{prox}
\newcommand{\indep}{\rotatebox[origin=c]{90}{$\models$}}

\usepackage{placeins}

\usepackage{xr}
\makeatletter

\hypersetup{
  colorlinks = true,
  allcolors = green!50!black,
  urlcolor = red!50!black,
}

\theoremstyle{theorem}
    
    \numberwithin{theorem}{section}
	\crefname{theorem}{Theorem}{Theorems}

\theoremstyle{definition}
    \newaliascnt{definition}{theorem}
    
    \aliascntresetthe{definition}
	\crefname{definition}{Definition}{Definitions}

\theoremstyle{corollary}
    \newaliascnt{corollary}{theorem}
    \newtheorem{corollary}[corollary]{Corollary}
    \aliascntresetthe{corollary}
	\crefname{corollary}{Corollary}{Corollarys}
    
\theoremstyle{proposition}
    \newaliascnt{proposition}{theorem}
    \newtheorem{proposition}[proposition]{Proposition}
    \aliascntresetthe{proposition}
	\crefname{proposition}{Proposition}{Propositions}
    
\theoremstyle{ex}
    \newaliascnt{ex}{theorem}
    
    \aliascntresetthe{ex}
	\crefname{ex}{Example}{Examples}
    
\theoremstyle{lemma}
    \newaliascnt{lemma}{theorem}
    \newtheorem{lemma}[lemma]{Lemma}
    \aliascntresetthe{lemma}
	\crefname{lemma}{Lemma}{Lemmas}
    
\theoremstyle{defprop}
    \newaliascnt{defprop}{theorem}
    
    \aliascntresetthe{defprop}
	\crefname{defprop}{Definition and Proposition}{Definition and Propositions}

\theoremstyle{remark}
    \newaliascnt{remark}{theorem}
    \newtheorem{remark}[remark]{Remark}
    \aliascntresetthe{remark}
	\crefname{remark}{Remark}{Remarks}
    
\theoremstyle{assumption}
    \newaliascnt{assumption}{theorem}
    
    \aliascntresetthe{assumption}
	\crefname{assumption}{Assumption}{Assumptions}
  
\crefname{figure}{Figure}{Figure}
\crefname{table}{Table}{Table}

\numberwithin{figure}{section}
\numberwithin{table}{section}
\numberwithin{equation}{section}
\numberwithin{theorem}{section}
\numberwithin{algorithm}{section}

\newcommand*{\addFileDependency}[1]{
\typeout{(#1)}
\@addtofilelist{#1}
\IfFileExists{#1}{}{\typeout{No file #1.}}
}\makeatother

\definecolor{spycolor}{RGB}{150,150,200}
\tikzset{
    rectspy/.default={lens={scale=3}, size=3cm},
    rectspy on/.style={#1,},
    rectspy/.style={
        draw=spycolor,
        connect spies,
        spy scope={
        every spy on node/.style={
            draw=spycolor,
            very thick,
            rectangle, 
            rectspy on,
        },
        every spy in node/.style={
            draw=spycolor,
            very thick,
            rectangle,
        },
        #1
        },
        spy connection path={\draw[spycolor, very thick] (tikzspyonnode) -- (tikzspyinnode);}
    }
}

\tikzset{
    sepbar/.style={
        very thick,
        black!30!white,
    }
}

\usepackage{enumitem}
\setlist[enumerate]{leftmargin=.5in}
\setlist[itemize]{leftmargin=.5in}

\title{Posterior-Variance-Based Error Quantification for Inverse Problems in Imaging}

\author{Dominik Narnhofer \thanks{Institute for Computer Graphic and Vision, Graz University of Technology. (\href{mailto:dominik.narnhofer@icg.tugraz.at}{dominik.narnhofer@icg.tugraz.at}, \href{mailto:pock@icg.tugraz.at}{pock@icg.tugraz.at}, \url{https://www.icg.tugraz.at/}).} \and Andreas Habring \thanks{Institute of Mathematics and Scientific Computing, University of Graz. (\href{mailto:andreas.habring@uni-graz.at}{andreas.habring@uni-graz.at}, \href{mailto:martin.holler@uni-graz.at}{martin.holler@uni-graz.at}, \url{https://mathematik.uni-graz.at/en/institute/research-areas/mathds/}).} \and Martin Holler\footnotemark[3] \and Thomas Pock\footnotemark[2]}

\usepackage{amsopn}

\newcommand\keywordsname{Key words}
\newcommand\keywordname{Key word}
\newcommand\MSCname{MSC codes}

\newenvironment{@abssec}[1]{%
     \if@twocolumn
       \section*{#1}%
     \else
       \vspace{.05in}\footnotesize
       \parindent .2in
         {\upshape\bfseries #1. }\ignorespaces 
     \fi}
     {\if@twocolumn\else\par\vspace{.1in}\fi}
     
\newenvironment{keywords}{\begin{@abssec}{\keywordsname}}{\end{@abssec}}

\newenvironment{MSCcodes}{\begin{@abssec}{\MSCname}}{\end{@abssec}}

\ifpdf
\hypersetup{
  pdftitle={Posterior-Variance-Based Error Quantification for Inverse Problems in Imaging},
  pdfauthor={D. Narnhofer, A. Habring, M. Holler and T. Pock}
}
\fi

\newcommand{\Df}{\mathrm{D}}
\newcommand{\RR}{\mathbb{R}}
\newcommand{\norm}[2][]{\|{#2}\|_{{#1}}}
\newcommand{\scp}[2]{\langle #1,#2 \rangle}

\begin{document}

\maketitle

\begin{abstract}
In this work, a method for obtaining pixel-wise error bounds in Bayesian regularization of inverse imaging problems is introduced. The proposed method employs estimates of the posterior variance together with techniques from conformal prediction in order to obtain coverage guarantees for the error bounds, %
without making any assumption on the underlying data distribution. It is generally applicable to Bayesian regularization approaches, independent, e.g., of the concrete choice of the prior. Furthermore, the coverage guarantees can also be obtained in case only approximate sampling from the posterior is possible. With this in particular, the proposed framework is able to incorporate any learned prior in a black-box manner. Guaranteed coverage without assumptions on the underlying distributions is only achievable since the magnitude of the error bounds is, in general, unknown in advance. Nevertheless, experiments with multiple regularization approaches presented in the paper confirm that in practice, the obtained error bounds are rather tight. For realizing the numerical experiments, also a novel primal-dual Langevin algorithm for sampling from non-smooth distributions is introduced in this work.

\end{abstract}

\begin{keywords}
Bayesian imaging, conformal prediction, uncertainty quantification, error control, learned priors, inverse problems
\end{keywords}

\begin{MSCcodes}
68U10 $\cdot$ 62F15 $\cdot$ 65C40 $\cdot$ 65C60 $\cdot$ 65J22
\end{MSCcodes}
\section{Introduction}
In this work, we provide a generally applicable method for solving inverse problems, and for obtaining statistical bounds on the error of the prediction. The bounds are proven to hold in probability via methods of conformal prediction \cite{romano2019conformalized}, without making a-priori assumptions on the underlying distributions and by only using a finite data set.

In inverse problems, the goal is to recover the true source signal $x \in \mathcal{X}$ from a corrupted observation $z \in \mathcal{Z}$ obtained from the so-called \emph{measurement} or \emph{forward} operator $\mathcal{F}:\mathcal{X}\rightarrow \mathcal{Z}$
\begin{equation}\label{eq:inverse_problem}
    z = \mathcal{F}(x)+\nu,
\end{equation}
where $\nu$ denotes the measurement noise. Since in most applications of interest the forward operator $\mathcal{F}$ is not invertible or at least severly ill-conditioned, \eqref{eq:inverse_problem} is not well-posed in general. In a probabilistic framework, $x$, $z$, and $\nu$ are replaced by the random variables $X$, $Z$, and $N$ and, consequently, \eqref{eq:inverse_problem} is modelled by a common distribution $\prob(x,z)$. In this setting, solving the inverse problem amounts to estimating the random variable $X$ based on an observation of $Z$, which is frequently done by invoking the posterior distribution computed using Bayes theorem
\begin{equation}\label{eq:bayes_thm}
\prob(x|z) = \frac{\prob(z|x)\prob(x)}{\prob(z)}
\end{equation}
respectively its negative logarithm
\[-\log(\prob(x|z))=-\log(\prob(z|x))-\log(\prob(x))+const.\]
While a model for the data likelihood $\prob(z|x)$ is easily found if the noise distribution is fixed, a lot of research has gone into modelling the prior distribution $\prob(x)$, see \Cref{sec:related_works_reg}. The present work, however, is directed at a different aim: Given any (approximate) model of the prior distribution, and a resulting method to obtain point estimates, we want to obtain error bounds with coverage guarantees for the point estimate that hold true even if the model of the prior distribution and/or the method to obtain point estimates are not 100\% accurate. 

To explain our approach, let us assume for simplicity $\Xc = \R$. Given a prediction $\hat{x}(Z)$ for $X$ and an i.i.d. sample $(X_i,Z_i)_{i=1}^m$, we want to obtain a method for predicting quantiles $\hat{S}_q = \hat{s}_q((X_i,Z_i)_{i=1}^m,Z)$ of the squared error $S=s(X,Z)=(X-\hat{x}(Z))^2$ satisfying
\begin{equation}\label{eq:intro_coverage}
\Prob[S\leq \hat{S}_q] \geq q
\end{equation}
where $q\in (0,1)$ can be defined arbitrarily by the user. Here, the probability is computed over $(X_i,Z_i)_{i=1}^m$ and $(X,Z)$, and \eqref{eq:intro_coverage} shall be satisfied without assumptions on the underlying distributions. Our approach to obtain the quantiles $\hat{S}_q$ builds on the relation between error and posterior variance. More precisely, given an estimator $\hat{t}(z)\approx \Var[X|Z=z]$, $\hat{T}=\hat{t}(Z)$ for the posterior variance, we first approximate the distribution of the $(S,\hat{T})$ on the sample $(X_i,Z_i)_{i=1}^m$, where we split the range of $\hat{T}$ into different bins $\tau_k = [t_k,t_{k+1})$ for $k=1,2,\ldots$ For each bin $\tau_k$, we then set $\hat{S}^{\tau_k}_q$ to be an approximation of the $q$-quantile of $\{S_i\;|\; \hat{t}(Z_i) \in \tau_k\}$. For this, we make use of conformal prediction techniques to ensure coverage (see \cite{romano2019conformalized} for the original work that motivated this part of our approach). In other words, $\hat{S}^{\tau_k}_q$ is an estimate of the $q$-quantile of $S|\hat{T}\in\tau_k$ based on a finite sample. For a new data point $Z$, the estimated $q$-quantile $\hat{S}_q$ such that \eqref{eq:intro_coverage} holds is then obtained as $\hat{S}_q = \hat{S}^{\tau_k}$, where $k$ is such that $\hat{t}(Z) \in \tau_k$. With this, our method acts as a regression from posterior variance to the squared error, where the regression is carried out pointwise on different variance bins (thus it is non-linear). Our coverage guarantees hold independently of the true distribution of $(X,Z)$ and even in the case that both $\hat{x}(Z)$ and $\hat{t}(Z)$ are sub-optimal estimators (in fact, they can even be arbitrary). The reason why this is possible is because we make \emph{no guarantees on the magnitude of the error bounds} (i.e., on the size of the interval of uncertainty). Naturally, the error bounds will be smaller, the better the prior $\prob(x)$ and the estimators $\hat{x}(Z)$ and $\hat{t}(Z)$. While our coverage guarantees hold in any case, it can only be observed after carrying out our method on a new datum, if the error bounds obtained for this datum are sufficiently small for practical use. Our experiments show that this is the case with different, reasonable priors.

For the estimation of statistical quantities, such as the posterior variance or the expected posterior, we rely on Langevin based Markov chain Monte Carlo sampling \cite{laumont2022bayesian}, where we introduce a novel variation of the Langevin algorithm applicable to non-smooth log-distributions. The contributions of our work can be summarized as follows.
\paragraph{Contributions}
\begin{itemize}
    \item We propose a novel method for error estimation for inverse problems in imaging. Based on a finite i.i.d. sample we estimate quantiles of the reconstruction error conditioned on an estimate of the posterior variance. 
    \item Coverage of the estimated quantiles is proven without making assumptions on the underlying distributions and by only invoking a finite data set by means of conformal prediction.
    \item Numerical experiments supporting our claims are conducted on multiple inverse problems as well as with different regularization methods. In particular we show that the error guarantees are satisfied in practice and provide additional empirical evidence demonstrating the strong relation between posterior variance and error.
      \item Within our numerical experiments, we propose a new approach for posterior sampling in the presence of a non-smooth convex functional based on the primal-dual algorithm.
    \item Furthermore, a novel method for evaluation of Langevin-like algorithms by means of a Markov Random Field is presented.
  
\end{itemize}

\section{Related Works}\label{sec:related_works}
\subsection{Regularization and Posterior Sampling}\label{sec:related_works_reg}
As the framework for solving inverse problems is well-established, most research nowadays revolves around the choice of the prior respectively regularizer $\prob(x)$ in \eqref{eq:bayes_thm}.
Two categories can be distinguished: In the past hand-crafted priors with certain assumptions on regularity such as sparsity \cite{rof92, sa01, chpo11, da04, ho20_ip_review} were commonly the preferred choice, while nowadays the field of inverse problems is mainly dominated by data driven approaches that either directly learn an inverse mapping \cite{ja08, le17, zb18} or attempt to learn a suitable prior.
The latter is frequently done by means of unrolling \cite{ham18, koef21, aga18} or as in recent works by the use of generative approaches \cite{bor17, na19, son19,habring2022generative}.
Priors, rather than penalty functionals, are of particular interest, since they offer the possibility of sampling due to their access to the posterior \cite{laumont2022bayesian}.
Although the sampling method itself is subject to ongoing research \cite{do21, pe16,ro99}, Langevin algorithms for Markov chain Monte Carlo sampling \cite{roberts1996exponential,durmus2019high,durmus2018efficient} are widely used either for inference \cite{laumont2022bayesian} or also in training \cite{zach2022computed, son19}.

\subsection{Uncertainty Quantification}
Uncertainty quantification is the task of quantifying the certainty or confidence of a prediction. Note, however, that there is a conceptual difference between uncertainty and prediction error since a very confident prediction can still be erroneous. While in the literature it has, e.g., been proposed to use the posterior variance as a measure of uncertainty \cite{zach2022computed}, a link to the reconstruction error is rarely found in existing works. It is precisely this link between uncertainty and error that we investigate in this work. Most works about uncertainty quantification in the literature either utilize Bayesian neural networks or posterior sampling using Markov chain Monte Carlo methods. In \cite{blundell2015weight} an algorithm for learning a distribution over the weights of a neural network is proposed. In \cite{kega17,narnhofer2021bayesian} the authors use Bayesian neural networks to investigate different types of uncertainty (aleatoric vs epistemic) and in \cite{rep19} the authors use hypothesis tests in a Bayesian framework to quantify the probability of the presence of certain structures in an image reconstruction. For a review of Bayesian uncertainty quantification of deep learning methods see \cite{abdar2021review}. In \cite{gal2016dropout} it is shown that dropout training in deep neural networks is equivalent to approximate Bayesian inference in deep Gaussian processes yielding tools to model uncertainty. In \cite{sc18} Markov chain dropout is used for deep learning in MRI reconstruction to investigate uncertainty and the authors also qualitatively relate the uncertainty to the reconstruction error. In \cite{ra20} the authors sample from the posterior using Monte Carlo sampling and modelling the gradient of the log prior distribution by a denoising auto encoder. In \cite{lu22} the authors propose to learn the reverse noise process in the context of MRI reconstruction allowing for sampling from the posterior. Closely related to our approach, in \cite{zach2022computed} the authors propose posterior sampling using Langevin dynamics combined with a generative prior to estimate the posterior variance.

\subsection{Conformal Prediction}
The proposed work also builds on ideas of conformal prediction and risk control, methods of obtaining distribution free confidence bounds for model predictions based on a finite training sample. Risk control describes methods, where an i.i.d. data set is used to bound the expected loss. The procedure builds on concentration inequalities, such as Hoeffding's inequality, which allow to bound an expected value based on a finite sample mean \cite{angelopoulos2021learn,angelopoulos2022image,medarametla2021distribution,angelopoulos2022conformal,bates2021distribution}. On the other hand, in conformal prediction, introduced in \cite{vovk2005algorithmic,shafer2008tutorial}, experience on a data set is used to estimate quantiles and/or confidence regions for a random variable, typically under the assumption of exchangeability of the sample \cite{tibshirani2019conformal,lei2018distribution,lei2015conformal,lei2013distribution,sadinle2019least,romano2019conformalized}. The present work can, indeed, be put into this framework: If the proposed method outputs a point prediction $\hat{x}$ and an upper bound for the squared error $\hat{s}_q$, then $[\hat{x}-\sqrt{\hat{s}_q}, \hat{x}+\sqrt{\hat{s}_q}]$ is a $q$ confidence interval for the ground truth. In terms of applications, our work is related most closely to \cite{angelopoulos2022image} where the authors propose a method predicting per-pixel confidence intervals for image regression. Contrary to our work, however, the authors guarantee a bounded expected rate of wrongly predicted pixel intervals in an image using risk control based on concentration inequalities. In the present work we bound the error probability for each pixel individually. From a theoretic viewpoint, our work is in the same spirit as \cite{romano2019conformalized}. There authors use conformal prediction to calibrate confidence intervals obtained by quantile regression. Distinct from both of the mentioned works we do not assume a given method for interval predictions, but only require access to a point predictor. Using data to estimate the distribution of this method's error, we implicitly obtain a set predictor. Moreover, we account for heteroscedasticity by conditioning the conformal prediction procedure on the given observation.

\section{Theoretical Results}
\label{sec:main}
In this section we first explain the proposed method in an abstract setting before elaborating on the application in imaging.

\subsection{Preliminaries}
We denote random variables as upper case letters and deterministic quantities or realizations of random variables as lower case letters. Random variables obtained as functions of other random variables are consequently denoted as the upper case letter of the function, e.g., we denote the posterior variance as $T=t(Z)$ with the function $t(z)=\Var[X\;|\;Z=z]$ evaluated at the random variable $Z$. Moreover, approximations/estimators are denoted with a hat, e.g., the estimator of the posterior variance $t$ is denoted as $\hat{t}$ and analogously for the random variables $T,\hat{T}$. Independence of two random variables $X,Y$ is denoted as $X\indep Y$. For a real valued random variable $Y$ with cumulative distribution function $F(y) = \Prob[Y\leq y]$, we denote for $q\in (0,1)$ the $q$-quantile as
\[y_q = \inf\{ y\in\R\; | \; F(y)\geq q\}.\]
Further for i.i.d. random variables $Y_1,\dots,Y_n$, we define the empirical quantile $\hat{Y}_q$ as the quantile of the empirical distribution function $\hat{F}(y) = \frac{1}{n}\sum_{i=1}^n \I_{Y_i\leq y}$,
\[\hat{Y}_q = \inf\{ y\in\R\; | \; \hat{F}(y)\geq q\}\]
which can be computed explicitly as $\hat{Y}_q = Y_{(\lceil qn \rceil)}$ with $Y_{(k)}$ being the k-th smallest value in $Y_1,\dots,Y_n$. We present a slight modification of a key result from conformal prediction, which can be found in different versions in \cite{vovk2005algorithmic,lei2018distribution,tibshirani2019conformal,romano2019conformalized}. In particular, the presented proofs are a generalization of \cite[Lemmas 1,2, Appendix]{romano2019conformalized} where we additionally allow for a random sample size.
\begin{lemma}\label{lem:cqr1}
Let $Y_1, Y_2,\dots,Y_{N+1}$ real valued random variables with $N\in\N$ a random sample size. Assume that $Y_1, Y_2,\dots,Y_{N+1}$ are i.i.d. with respect to the probability measure $\Prob[\;.\;|\; N=n]$ for every $n\in\N$, $n\geq 1$. Then for any $q\in (0,1)$ and $n\geq 1$,
\[\Prob[Y_{N+1}\leq \hat{Y}_q\;|\; N=n]\geq q\]
where $\hat{Y}_q = Y_{(\lceil(N+1)q\rceil)}$ is the empirical q-quantile of $Y_1, Y_2,\dots,Y_{N+1}$.
\end{lemma}
\begin{remark}
    The sample size is set to $N+1$ instead of $N$ in order to enable a consistent notation for later results where the first $N$ random variables will be the estimation data and the $(N+1)$-th sample a new data point.
\end{remark}
\begin{proof}
    By definition of the empirical quantile it is always true that
    \[q\leq \hat{F}(\hat{Y}_q).\]
    As a consequence we find for fixed $n\geq 1$
    \begin{equation}
        \begin{aligned}
            q\leq\E[ \hat{F}(\hat{Y}_q)\;|\;N=n] 
            =& \E[ \frac{1}{N+1}\sum\limits_{i=1}^{N+1}\I_{Y_i\leq\hat{Y}_q}\;|\;N=n] 
            =\E[ \frac{1}{n+1}\sum\limits_{i=1}^{n+1}\I_{Y_i\leq\hat{Y}_q}\;|\;N=n] \\
            =&\frac{1}{n+1}\sum\limits_{i=1}^{n+1}\Prob[Y_i\leq\hat{Y}_q\;|\;N=n] 
            =\Prob[Y_{n+1}\leq\hat{Y}_q\;|\;N=n]\\
            =&\Prob[Y_{N+1}\leq\hat{Y}_q\;|\;N=n].
        \end{aligned}
    \end{equation}
\end{proof}
In the preceding Lemma we have proven the intuitive assertion that the empirical quantile of $Y_1,\dots Y_{N+1}$ has guaranteed coverage for any of the $Y_i$. In the subsequent result, on the other hand, we conclude, that by correcting the quantile by a factor $(1+\frac{1}{N})$ coverage is also guaranteed for a new unseen i.i.d random variable.
\begin{lemma}\label{lem:cqr}
Let $Y_1, Y_2,\dots,Y_{N},Y$ be real valued random variables with $N\in\N$ random such that $Y_1, Y_2,\dots,Y_{N},Y$ are i.i.d. with respect to the probability measure $\Prob[\;.\;|\; N=n]$ for every $n\in\N$, $n\geq 1$. Assume additionally that $(1+\frac{1}{n})q\leq 1$, then
\[\Prob[Y\leq \hat{Y}_{(1+\frac{1}{N})q}\;|\;N=n]\geq q\]
where now $\hat{Y}_{(1+\frac{1}{N})q} = Y_{(\lceil(N+1)q\rceil)}$ is the empirical quantile of $Y_1, Y_2,\dots,Y_{N}$.
We call $\hat{Y}_{(1+\frac{1}{N})q}$ the \emph{conformalized} q-quantile.
\end{lemma}
\begin{proof}
    Let us denote $Y_{N+1}=Y$ and $Y_{(k,m)}$ the k-th largest value in $Y_1,\dots, Y_m$. Note that for $k$ such that $1\leq k\leq n$, $Y_{n+1}\leq Y_{(k,n)}$ if and only if $Y_{n+1}\leq Y_{(k,n+1)}$. Thus, since by assumption $\lceil(n+1)q\rceil\leq n$,
    \begin{equation}
        \begin{aligned}
            \Prob[Y_{N+1}\leq Y_{(\lceil(N+1)q\rceil,N)}\;|\;N=n] = \Prob[Y_{N+1}\leq Y_{(\lceil(N+1)q\rceil,N+1)}\;|\;N=n]
        \end{aligned}
    \end{equation}
    The result follows by applying \Cref{lem:cqr1} to the right-hand side of the equation.
\end{proof}

\subsection{The Proposed Method}
Let us assume for now that the signal space $\Xc=\R$. Let $(X_i,Z_i)_{i=1}^m\subset \Xc \times \Zc$ be an i.i.d. sample and $(X,Z)\sim (X_i,Z_i)$. Further, let $\hat{x}:\Zc\rightarrow \Xc$ be a prediction function for the ground truth $X$, e.g., $\hat{x}(z)\approx \E[X\;|\;Z=z]$, an approximation of the expected posterior. Define the squared error of the prediction $S = s(X,Z) = (\hat{x}(Z)-X)^2$ and accordingly $S_i = s(X_i,Z_i)$. A naive approach to obtain guaranteed error bounds would be to simply apply \Cref{lem:cqr} to the random variables $S_i, S$. Note, however, that this would yield an error bound which is independent of the given observation $Z$. In particular for heteroscedastic data this will result in inefficient error estimates. While in \cite{romano2019conformalized} the authors propose to account for heteroscedasticity by applying \Cref{lem:cqr} to an interval predictor which incorporates heteroscedasticity already, we propose a different approach, namely integrating information about $Z$ in the form of conditional probabilities. Denote an approximation of the posterior variance as $\hat{t}(z) \approx \Var[X|Z=z]$, $\hat{T} = \hat{t}(Z)$ and $\hat{T}_i = \hat{t}(Z_i)$. Ideally, we would apply \Cref{lem:cqr} to the random variable $S|\hat{T}=\tau$ for any fixed value of $\tau \in \R$ yielding guarantees on the error for any fixed value of the approximated variance. This way we would make use of all the information contained in $\hat{T}$. But since the estimation of a quantile of the random variable $S$ conditioned on any point-value of $\hat{T}$ is unfeasible in practice, where only a finite amount of data is available, a relaxation is needed. This is achieved by partitioning $[0,\infty)$ into disjoint intervals $\tau_k = [t_k,t_{k+1})$, $k=0,1,2,\dots$, and, for any given interval $\tau_k$, considering the error conditioned on $\hat{T}\in\tau_k$. To realize this, for any given interval $\tau_k$, first define the estimated conditional error quantile
\begin{equation}\label{eq:quantile_estimator}
    \hat{S}^{\tau_k}_q \coloneqq (1+\frac{1}{N_{\tau_k}})q-\text{empirical quantile of } \{S_i\;|\; \hat{T}_i \in \tau_k\}.
\end{equation}
where $N_{\tau_k}=|\{S_i\;|\; \hat{T}_i \in \tau_k\}|$. In the case that $(1+N_{\tau_k})q>N_{\tau_k}$, we set $\hat{S}^{\tau_k}_q$ to be the essential supremum of the random variable $S$. Given a new observation $Z$, the final method can then be described as follows:
\begin{enumerate}
    \item Compute $\hat{x}(Z)$ and $\hat{t}(Z)$.
    \item Pick $\tau_k$, such that $\hat{t}(Z)\in \tau_k$.
    \item Compute the error quantile estimator $\hat{S}_q$ as $\hat{S}^{\tau_k}_q$.
\end{enumerate}
\begin{remark}\label{rmk:choice_T_x_bar}
The following results do not rely on exact knowledge of the posterior expectation or variance, which is why we already introduced $\hat{X}$ and $\hat{T}$ as approximations above. As a matter of fact, \Cref{prop:coverage,cor:coverage} below are entirely independent of the particular choices of the functions $\hat{x}(z)$ and $\hat{t}(z)$. In particular, all results remain true despite using heuristics within the computations of the approximations of posterior expectation and variance, such as the choice of the thinning parameter explained in \Cref{rmk:thinning} below. While the theory is not affected by the choices of $\hat{x}$ and $\hat{t}$, the tightness, and thus quality, of the error estimates will crucially depend on the predictive capability of $\hat{T}$ with respect to the error $S$. Our choice is motivated by the high similarity between the squared error to the expected posterior $(X-\E[X\;|\;Z])^2$ and the posterior variance $\Var[X\;|\;Z]=\E[(X-\E[X\;|\;Z])^2]\;|\;Z]$, which is, in fact, the conditional expectation of the former.
\end{remark}
As a direct consequence of \Cref{lem:cqr}, we obtain the following coverage guarantee.
\begin{proposition}\label{prop:coverage}
Let $(X_i,Z_i)_{i=1}^m\subset \Xc\times \Zc$ and $(X,Z) \in \Xc\times \Zc$ be i.i.d. Let $S=s(X,Z) = (X-\hat{x}(Z))^2$ and $\hat{T}=\hat{t}(Z)$ and assume that $S$ is bounded from above almost surely. Assume further that $\Prob[\hat{T}\in\tau_k]>0$, then the estimated conditional error quantile satisfies
\[\Prob[S\leq \hat{S}^{\tau_k}_q\;|\; \hat{T}\in \tau_k]\geq q\]
\end{proposition}
\begin{proof}
Fix $\tau_k$ and let $(S_{i_j})_{j=1}^{N_{\tau_k}}$ be the set of all $S_i$ for which $\hat{T}_i\in \tau_k$. That is,
\begin{align}
    i_1 = \min\{1\leq i\leq m\;|\;\hat{T}_i\in\tau_k\},\quad i_{j+1} = \min\{i_j< i \leq m\;|\;\hat{T}_i\in\tau_k\}.
\end{align}
In a sequence $(S_i,\hat{T}_i)_{i>m}$, of i.i.d. copies of $(S,\hat{T})$ let $S_{i_{N_{\tau_k}+1}}$ be defined as the next sample with variance $\hat{T}\in\tau_k$ in this sequence, i.e.,
\[i_{N_{\tau_k}+1} = \min\{m< i\;|\;\hat{T}_i\in\tau_k\}.\]
Note that by the Borel-Cantelli Lemma, $i_{N_{\tau_k}+1}<\infty$ a.s., since $\Prob[\hat{T}\in\tau_k]>0$. For the distributions of these random variables we find for any $n\geq 1$, $j\leq n$ and $s\in\R$ by the law of total probability
\begin{equation}
\begin{aligned} \label{eq:intex_to_bin_change}
    \Prob[S_{i_j}\leq s\;|\; N_{\tau_k}=n] = \sum_{l=1}^m \Prob[S_{i_j}\leq s\;|\; i_j=l, N_{\tau_k}=n]\;\Prob[i_j=l\;|\; N_{\tau_k}=n]\\
    =\sum_{l=1}^m \Prob[S_l\leq s\;|\; \hat{T}_i\in \tau_k \text{ for $i=l$ and for $(n-1)$ different $i\in\{1,\dots,m\}\setminus\{l\}$} ]\\ \;\Prob[i_j=l\;|\; N_{\tau_k}=n]\\
    =\sum_{l=1}^m \Prob[S_l\leq s\;|\; \hat{T}_l\in \tau_k ]\;\Prob[i_j=l\;|\; N_{\tau_k}=n]\\
    =\sum_{l=1}^m \Prob[S\leq s\;|\; \hat{T}\in \tau_k ]\;\Prob[i_j=l\;|\; N_{\tau_k}=n]\\
    = \Prob[S\leq s\;|\; \hat{T}\in \tau_k ]
\end{aligned}
\end{equation}
and the same result can be obtained for $j=n+1$ via an analogous computation. Further we can deduce independence as follows. Let $s_1,\dots,s_{n+1}\in\R$,
\begin{equation}
\begin{aligned}
    \Prob[\forall j=1,2,\ldots,n+1: S_{i_j}\leq s_j\;|\; N_{\tau_k}=n] \\
    = \sum_{l_1<l_2<\dots<l_n\leq m<l_{n+1}}\Prob[\forall j:\;S_{i_j}\leq s_j\;|\; \forall j:\;i_j=l_j,\;N_{\tau_k}=n]\;\Prob[\forall j:\;i_j=l_j, \;|\;N_{\tau_k}=n]\\
    = \sum_{l_1<l_2<\dots<l_n\leq m<l_{n+1}} \Prob[\forall j:\;S_{l_j}\leq s_j\;|\; \hat{T}_i\in\tau_k \Leftrightarrow i \in \{l_1,\ldots,l_n\}]\;\Prob[\forall j:\;i_j=l_j\;|\;N_{\tau_k}=n]\\
    = \sum_{l_1<l_2<\dots<l_n\leq m<l_{n+1}} \Prob[\forall j:\;S_{l_j}\leq s_j\;|\; \forall j:\;\hat{T}_{l_j}\in\tau_k]\;\Prob[\forall j:\;i_j=l_j\;|\;N_{\tau_k}=n]\\
    = \sum_{l_1<l_2<\dots<l_n\leq m<l_{n+1}} \Prob[\forall j:\;S_{j}\leq s_j\;|\; \forall j:\;\hat{T}_{j}\in\tau_k]\;\Prob[\forall j:\;i_j=l_j\;|\;N_{\tau_k}=n]\\
    = \Prob[\forall j:\;S_{j}\leq s_j\;|\; \forall j:\;\hat{T}_{j}\in\tau_k]
    = \prod\limits_{j=1}^{n+1}\Prob[S_{j}\leq s_j\;|\; \hat{T}_{j}\in\tau_k]
    = \prod\limits_{j=1}^{n+1}\Prob[S_{i_j}\leq s_j\;|\; N_{\tau_k}=n]\\
\end{aligned}
\end{equation}
where the last step follows from \eqref{eq:intex_to_bin_change}.
Thus, we can apply \Cref{lem:cqr} to $(S_{i_j})_{j=1}^{N_{\tau_k}}$, $S_{i_{N_{\tau_k}+1}}$ and find for $(1+n)q<n$
\begin{equation}
\begin{aligned}
q\leq \Prob[S_{i_{N_{\tau_k}+1}}\leq \hat{S}^{\tau_k}_q\;|\; N_{\tau_k}=n]\\
= \sum\limits_{l>m} \Prob[S_{i_{N_{\tau_k}+1}}\leq \hat{S}^{\tau_k}_q\;|\; i_{N_{\tau_k}+1}=l,\;N_{\tau_k}=n]\;\Prob[i_{N_{\tau_k}+1}=l\;|\; N_{\tau_k}=n]\\
= \sum\limits_{l>m} \Prob[S_l\leq \hat{S}^{\tau_k}_q\;|\; i_{N_{\tau_k}+1}=l,\;N_{\tau_k}=n]\;\Prob[i_{N_{\tau_k}+1}=l\;|\; N_{\tau_k}=n]\\
= \sum\limits_{l>m} \Prob[S\leq \hat{S}^{\tau_k}_q\;|\; \hat{T}\in\tau_k,\;N_{\tau_k}=n]\;\Prob[i_{N_{\tau_k}+1}=l\;|\; N_{\tau_k}=n]\\
= \Prob[S\leq \hat{S}^{\tau_k}_q\;|\; \hat{T}\in\tau_k,\;N_{\tau_k}=n].
\end{aligned}
\end{equation}
In the case that $(1+n)q\geq n$ (in particular $n=0$), we find
\[q\leq 1=\Prob[S\leq \hat{S}^{\tau_k}_q\;|\;\hat{T}\in\tau_k,\;N_{\tau_k}=n]\]
as well by definition of $\hat{S}^{\tau_k}$ as the essential supremum of $S$. As a result, the law of total probability yields
\begin{equation}
\begin{aligned}
\Prob[S\leq \hat{S}^{\tau_k}_q\;|\;\hat{T}\in\tau_k] = \sum\limits_{n=0}^\infty \underbrace{\Prob[S\leq \hat{S}^{\tau_k}_q\;|\;\hat{T}\in\tau_k,\;N_{\tau_k}=n]}_{\geq q}\Prob[N_{\tau_k}=n\;|\;\hat{T}\in\tau_k]\geq q.
\end{aligned}
\end{equation}
\end{proof}
\begin{remark}
    Assuming $S$ to be bounded a.s. in \Cref{prop:coverage} is only necessary for the definition of $\hat{S}^{\tau_k}_q$ in the case that $(1+N_{\tau_k})q>N_{\tau_k}$ as the essential supremum of $S$. This condition is trivially satisfied for the application in imaging where a pixel's color or gray scale value is in $[0,1]$ or $[0,255]$.
\end{remark}
We conclude coverage without the conditioning on a specific $\tau_k$.
\begin{corollary}\label{cor:coverage}
With the conditions of \Cref{prop:coverage}, define $\hat{S}_q = \hat{S}^{\tau_k}_q$ where $\tau_k$ is such that $\hat{T}\in\tau_k$, then
\[\Prob[S\leq \hat{S}_q]\geq q\]
\end{corollary}
\begin{proof}
Applying \Cref{prop:coverage} and the law of total probability we can compute
\begin{align*}
    \Prob[S\leq \hat{S}_q] &= \sum_{k=0}^\infty \Prob[S\leq \hat{S}_q\;|\;\hat{T}\in\tau_k]\;\Prob[\hat{T}\in\tau_k]\\
    &= \sum_{k=0}^\infty \Prob[S\leq \hat{S}^{\tau_k}_q\;|\; \hat{T}\in\tau_k]\;\Prob[\hat{T}\in\tau_k]
    \geq \sum_{k=0}^\infty q\;\Prob[\hat{T}\in\tau_k] = q.
\end{align*}
\end{proof}

\subsection{Application to Bayesian Imaging}
In the previous section we introduced the proposed method for estimating the prediction error in a general setting. In this section we explain the specific application to inverse problems in imaging.

In order to apply our method to high dimensional image data, we model the samples $X_i$ as the gray scale values of individual image pixels and $Z_i$ as the given observation of the image. That is, our method is applied to each image pixel separately. The method consists of two parts. First for a given a sample $(x_i,z_i)_{i}$, we compute $(s_i,\hat{t}_i)$ and afterwards, for new data $z$, $\hat{t}(z)$ is computed and, for $\tau_k$ such that $\hat{t}(z)\in\tau_k$, the empirical quantile $\hat{s}^{\tau_k}_q$ is determined. The procedure is depicted in \Cref{algo:error_estimation}.

\begin{algorithm}
\setstretch{1.15}
\caption{Error estimation.}\label{algo:error_estimation}
\textbf{Input:} Observation $z$, random sample $(x_i,z_i)_{i=1}^m$, variance bins $(\tau_k)_k$, confidence level $q\in(0,1)$.\\
\textbf{Output:} Point prediction $\hat{x}$, error estimate $\hat{s}_q$.
\begin{algorithmic}[1]
\For{$i=1,2,\dots,m$}
\State $\hat{x}_{i} = \hat{x}(z_i)$
\State $\hat{t}_i = \hat{t}(z_i)$
\State $s_i = (\hat{x}_i-x_i)^2$
\EndFor
\State $\hat{x}=\hat{x}(z)$
\State $\hat{t}=\hat{t}(z)$
\State Pick bin $\tau_k$ with $\hat{t}\in \tau_k$.
\State $n_{\tau_k} = \left|\left\{i\;\middle|\; \hat{t}_i\in\tau_k\right\}\right|$
\State Compute $\hat{s}_q$ according to \eqref{eq:quantile_estimator}.
\end{algorithmic}
\end{algorithm}

\begin{figure}[htb]
\includegraphics[width= \linewidth]{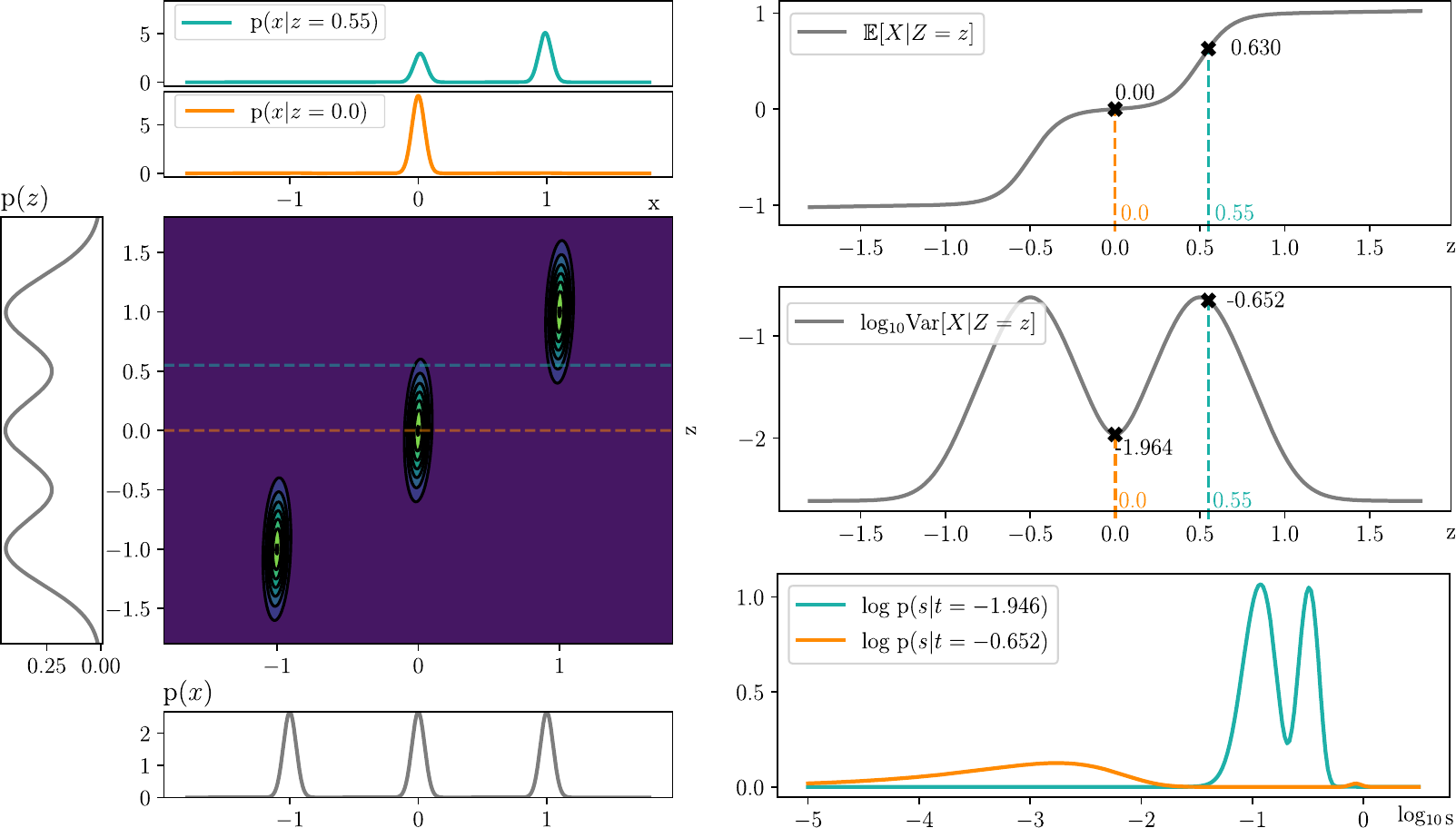}
\caption{Left: Joint distribution $\prob(x,z)$ with marginal distributions $\prob(x)$, $\prob(z)$ and posterior distributions $\prob(x|z)$. Right: Posterior expectation $\E[X|Z=z]$ and posterior variance $t=\Var[X|Z=z]$ as functions of $z$ as well as the conditional distribution of the error $p(s|t)$ for specific instances of $t$. As in later usage, posterior variance and error are already presented in logarithmic scaling.}
\label{fig:PDF}
\end{figure}

\subsection{Toy example in 1D}
Let us consider a toy example with known, continuous probability distributions. In this case we can work with exact quantiles instead of empirical estimates using conformal prediction.%
Let $X$ be a uni-variate mixture of Gaussian distributions. More precisely, define the density of $X$ as $\prob(x) = \sum_{k=1}^K\alpha_k \mathcal{N}(c_k,\sigma^2_{x_k})$ with $\alpha_k>0$ and $\sum_{k=1}^K\alpha_k=1$ and let $Z = X + N$, with $N \sim \mathcal{N}(0,\sigma^2_z)$. 
The joint density of $(X,Z)$ reads as
\begin{equation}
    \prob(x,z) = \sum_{k=1}^K\alpha_k \mathcal{N}(\mu_k,\Sigma_k), 
    \label{eq:joint}
\end{equation}
with
\begin{align*}
\mu_k &=
\begin{pmatrix} 
c_k \\ 
c_k  
\end{pmatrix},\quad
\Sigma_k=
\begin{pmatrix} 
\sigma^2_{x_k} & \sigma^2_{x_k} \\ 
\sigma^2_{x_k} & \sigma^2_{x_k} +\sigma^2_z 
\end{pmatrix},
\end{align*}
and is visualized in \Cref{fig:PDF} for $(c_1,c_2,c_3) = (-1,0,1)$, $\sigma^2_{x_k} = 0.05^2$, $\alpha_k = \frac{1}{3}$ for all $k$ and $\sigma^2_{z} = 0.3^2$. There we show the joint distribution of $(X,Z)$ as well as the corresponding marginal distributions. We additionally plot the resulting posterior distribution for observations of $z=0.55$ and $z=0.0$, the posterior expected value, and the posterior variance as well as the error. The latter two are presented in logarithmic scale.
In this toy example we can explicitly compute the joint distribution of error and variance $(S,T)$ using a change of variables. The result can be found in \Cref{fig:cum1D}, where on the left we plot the joint distribution and on the right the cumulative conditional distribution with respect to $S$, $(s,t)\mapsto \Prob[S\leq s | T=t]$ is shown. The dashed line indicates the 0.9 quantile of the conditional error distribution. An application of the proposed method in this ideal scenario reads as follows: Given an observation $Z=z$, compute the point prediction $\hat{x}(z)$ and posterior variance $t(z)$. Then in \Cref{fig:cum1D} the $0.9$ quantile of the error is obtained as the value of $s$ where the dashed quantile line intersects the vertical line at $t=t(z)$. While this example is mathematically well-behaved, in imaging applications, neither are statistics of the posterior distribution mathematically tractable, nor do we have access to an unlimited amount of data. The former forces us to make use of sampling techniques while the latter demands the use of conformal prediction in order to retain theoretical guarantees. In \Cref{fig:cum1Dempiric} we illustrate the conformal prediction strategy for the toy example. We estimate the distribution empirically from $2\cdot 10^6$ synthetically generated samples. In the right figure we added the exact conditional quantiles as well as the estimation based on conformal prediction, demonstrating the accuracy of the approximation.

\begin{figure}[htb]
\centering
\includegraphics[width= 0.7\linewidth]{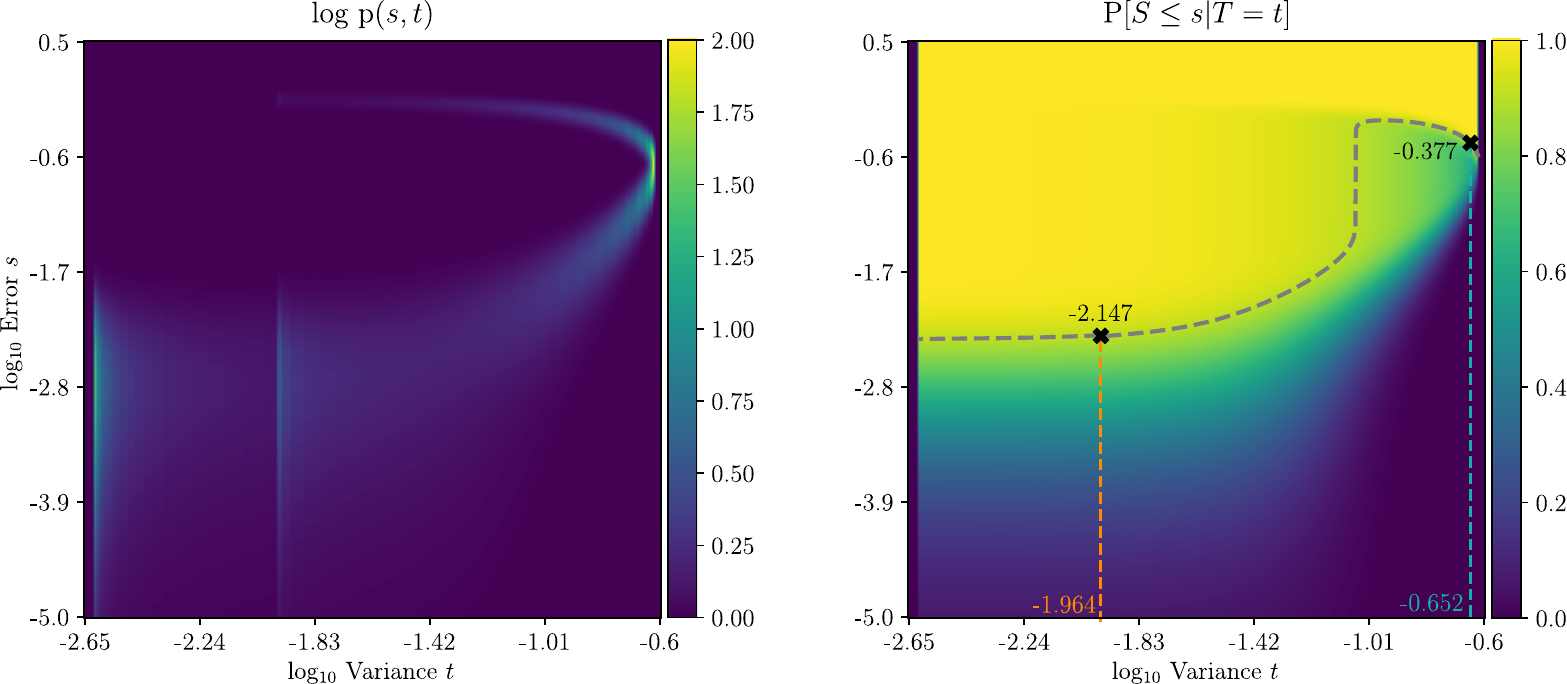}
\caption{Joint log-distribution $\log\prob(s,t)$ (left) and cumulative conditional distribution over error (right). The gray dashed line indicates the conditional 0.9 quantile.}
\label{fig:cum1D}
\end{figure}

\begin{figure}[htb]
\centering
\includegraphics[width= 0.7\linewidth]{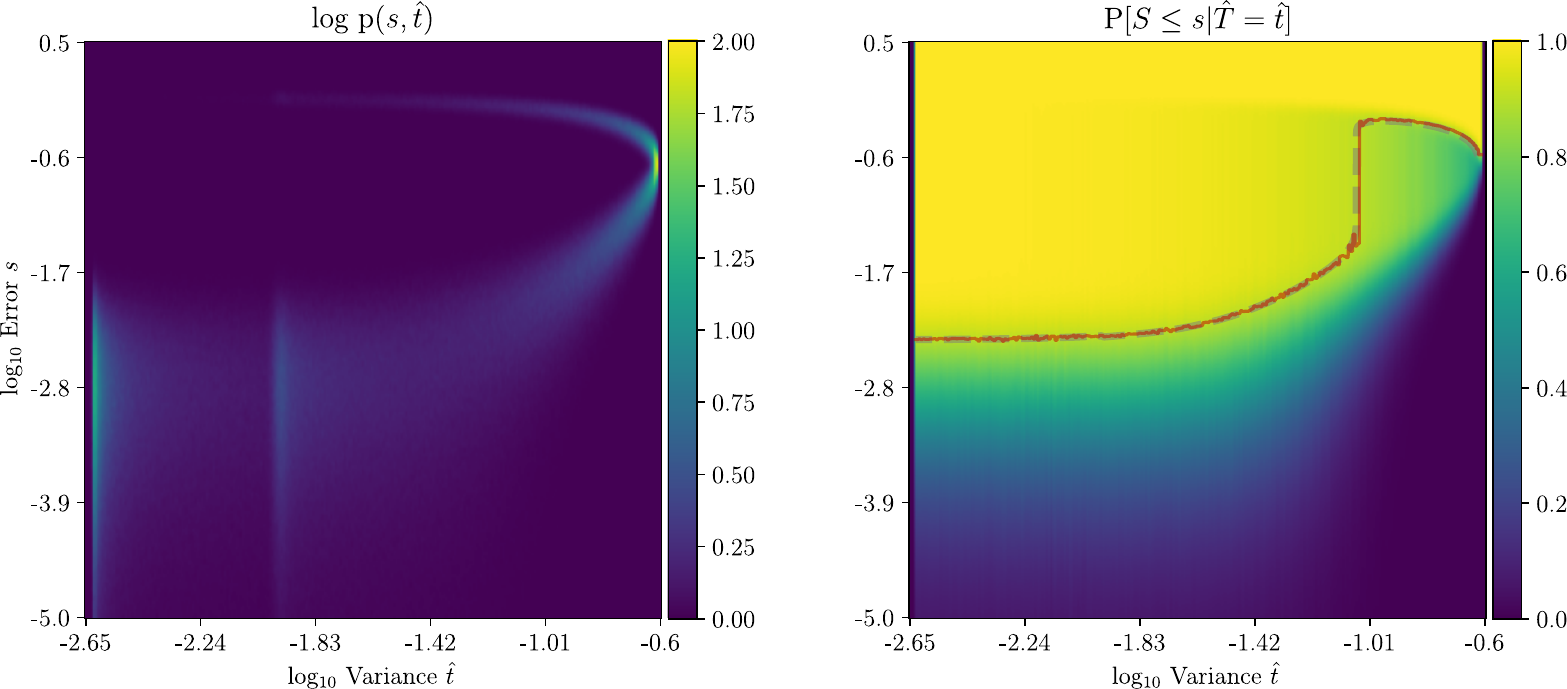}
\caption{Empirical joint log-distribution $\log\prob(s,t)$ (left) and empirical cumulative conditional distribution over error (right). The dashed gray line indicates the exact conditional 0.9 quantile of the distribution, the red lines indicates the respective estimated conformalized quantile.}
\label{fig:cum1Dempiric}
\end{figure}

\section{Algorithms for the Computation of Posterior Expectation and Variance}
Within our methods we need to compute approximations of posterior expectation and variance. For the typical models of image distributions, an analytical computation is not feasible. Hence, we are forced to compute statistics based on samples. In the following we introduce algorithms for sampling from distributions $\pi(x)$ from the exponential family $\pi(x) \propto \exp{\left(-E(x) \right)}$, on $\RR^d$, $d\geq 1$ with a potential (or energy) $E(x)$. In particular in the case of posterior sampling from a Bayesian model
\[E(x) = -\log(\prob(z|x)) - \log(\prob(x)) = -\log(\prob(x|z)) + const.\]
which is equal to the log-posterior up to an unknown normalization constant. The presented algorithms yield samples $(x_k)_{k=1}^K$ from to the distribution $\pi(x) = -\log(\prob(x|z))$ based on the observation $z$. Using this samples, as point prediction and variance approximation we will use the sample mean and sample variance, respectively, i.e., 
\begin{equation}\label{eq:used_estimators}
    \hat{x}(z)=\frac{1}{K}\sum\limits_{k=1}^K x_k,\quad \hat{t}(z)=\frac{1}{K}\sum\limits_{k=1}^K (x_k-\hat{x}(z))^2.
\end{equation}
After dealing with differentiable energies in \Cref{ssec:ULA}, in \Cref{ssec:PDLV}, we propose a novel sampling method called unadjusted Langevin primal-dual algorithm (ULPDA), which combines Langevin dynamics with a primal-dual algorithm. Further, in \Cref{sec:markov_random_fields}, we introduce Belief Propagation (BP) on Markov random fields (MRF) which is later used as a comparison benchmark for the Langevin algorithm.

\subsection{Langevin Algorithms for Differentiable Energies}\label{ssec:ULA}
Assume that the energy $E(x)$ is differentiable with Lipschitz continuous gradient, i.e., there exists a constant $L \geq 0$ such that
\[
\norm[2]{\nabla E(x) - \nabla E(y)} \leq L \norm[2]{x-y},
\]
for all $x,y \in \RR^d$. Langevin based sampling algorithms are discretizations of the over-damped Langevin diffusion equation~\cite{legy97,roberts1996exponential},
\begin{equation}
   \rm{d}X_t = -\nabla E(X_t) + \sqrt{2} \rm{d}B_t,  
   \label{eq:LVcont}
\end{equation}
where $\rm{d}B_t$ is a $d$-dimensional Brownian motion. The stationary distribution of $(X_t)_t$ is the distribution $\pi$ itself \cite{roberts1996exponential}. The striking advantage of the Langevin diffusion equation over classical density-based sampling algorithms is that only information about gradients is required and, hence, no normalization constant is needed, that is, it is sufficient to have knowledge of the distribution $\pi$ up to a multiplicative scalar factor. Intuitively, in \eqref{eq:LVcont} a gradient ascent towards a mode of $\pi$ is performed but the declination injected by the Brownian motion ensures that generated samples cover the entire space exactly with density $\pi(x)$. The scaling factor $\sqrt{2}$ balances the deterministic force induced by the gradient of the log distribution and the stochastic force induced by the Brownian motion.

For a numerical simulation of the stochastic Langevin diffusion equation, the simplest approach is the Euler-Maruyama discretization~\cite{klo13} which is rooted in the classical forward-Euler method. This yields an iterative algorithm, known as the Unadjusted Langevin algorithm, which essentially resembles a gradient descent on the energy $E(x)$ combined with noise (ULA)~\cite{durmus2018efficient}, see \Cref{algo:unaj_langevin}, where
\begin{algorithm}
\setstretch{1.15}
\caption{Unadjusted Langevin Algorithm (ULA)}\label{algo:unaj_langevin}
\textbf{Input:} Initial values $x_0 \in \R^d$, and step sizes $\tau_k > 0 $.\\
\textbf{Output:}  Sequence $\{x_k\}_{k\geq 0}$.
\begin{algorithmic}[1]
      \For{$k=0,\ldots,K-1$}   
      \State $x_{k+1} = x_k - \tau_k \nabla E(x_k) + \sqrt{2 \tau_k} \xi_k $
      \EndFor
\end{algorithmic}
\end{algorithm}
$\xi_k \sim \mathcal{N}(0,\Id_d)$ is i.i.d. from a $d$-dimensional Gaussian distribution. It has been shown that the Markov chain $\{x_k\}_{k\geq 0}$ generated by the ULA converges to a stationary distribution $\pi_\tau(x)$ and the target distribution $\pi(x)$ is obtained for $\tau_k \to 0$ and $\sum_k \tau_k = \infty$ \cite{lamberton2002recursive,lamberton2003recursive}. Conversely, for $\tau_k > 0$, the stationary distribution of $\{x_k\}_{k\geq 0}$ contains some sort of bias and hence remains \emph{unadjusted}. By the inclusion of an additional Metropolis test, the ULA can be turned into the Metropolis adjusted Langevin algorithm (MALA)~\cite{gar96}, whose stationary distribution is $\pi$, even in case that $\tau_k$ remains strictly bounded away from $0$.

\subsection{Primal-Dual Langevin Algorithm for Non-differentiable Energies}\label{ssec:PDLV}
We will now investigate the case of a non-differentiable potential with the aim of developing a Langevin based sampling algorithm which is also applicable, e.g., for potentials 
such as those incorporating TV-regularization, see \Cref{sec:tvl2_denoising}. A remedy making ULA applicable nonetheless is smoothing the potential $E(x)$, either by directly replacing any non-smooth function by a smooth approximation or by considering other model-based smoothness approaches such as the Moreau envelope~\cite{durmus2018efficient}, whose gradient can be computed efficiently whenever the proximal map can be computed efficiently. However, we propose a different approach. Let us consider the general case of an energy of the form $E:\R^d\rightarrow\R$,
\begin{equation}\label{eq:pd_general_energy}
E(x) = \sup_{p\in\R^{d'}} \scp{D x}{p} + g(x) - f^*(p).
\end{equation}
with $D: \R^d \to \R^{d'}$ a linear operator with operator norm $L = \norm{D}$, $f,g$ proper, convex, and lower-semicontinuous and $f^*$ the convex conjugate of $f$. Potentials of this type can be efficiently minimized by first-order primal-dual algorithms~\cite{chpo11}. Interestingly, the flexibility of the step size condition in the primal-dual algorithm opens up the possibility to turn the primal-dual algorithm into a primal algorithm ($\tau_k \to 0, \; \sigma_k \to \infty$), minimizing  a particular Moreau-envelope of the primal problem, or into a dual algorithm ($\tau_k \to \infty, \; \sigma_k \to 0$), maximizing a particular Moreau-envelope of the dual problem. Therefore, it is quite natural to propose an Unadjusted Langevin Primal-Dual Algorithm (ULPDA), which is a primal-dual algorithm with small primal step size and the Gaussian noise injected into the primal variable. The algorithm is summarized in \Cref{algo:pd_langevin}. 
\begin{algorithm}
\setstretch{1.15}
\caption{Unadjusted Langevin Primal-Dual Algorithm (ULPDA)}\label{algo:pd_langevin}
\textbf{Input:} Initial values $x_0 \in \R^d$, $p_0 \in \R^{d'}$, and step sizes $\sigma_k \tau_k L^2 \le 1 $, $\theta_k \in [0,1]$, $L = \|D\|$\\
\textbf{Output:}  Sequence $\{x_k\}_{k\geq 0}$.
\begin{algorithmic}[1]
      \For{$k=0,\ldots,K-1$}
      \State \begin{equation}\label{eq:algo_pb}
        \begin{cases}
            \bar p_{k} & = p_{k} + \theta_k( p_{k} -  p_{k-1}) , \\
            x_{k+1}& = \prox_{\tau_k g} (x_k - \tau_k D^* \bar p_k )+ \sqrt{2 \tau_k} \xi_k, \\
            p_{k+1} &= \prox_{\sigma_k f^\ast} (p_k + \sigma_k D x_{k+1} ).
            \end{cases}
      \end{equation}
      \EndFor
\end{algorithmic}
\end{algorithm}

\begin{remark}\label{rmk:thinning}\
\begin{itemize}
\item Consecutive iterates $(x_k,x_{k+1})$ from Langevin algorithms are not independent in general, which is why it is recommended to compute the approximations in \eqref{eq:used_estimators} based on the sample $(x_{kH})_{k}$, instead of $(x_k)_k$, where $H\in\N$, $H>>1$ defines the number of skipped samples in order to obtain a thinned version of the Markov chain. For sufficiently large $H$, it can be assumed that after thinning $(x_{kH},x_{kH+H})$ are approximately i.i.d. This, however, leads to a vast increase of computational effort. Throughout our experiments we set the thinning length to $H=1$. As explained in \Cref{rmk:choice_T_x_bar} the theoretical guarantees are not affected by heuristics in this part of the method. Moreover, an empirical justification of this choice can be found in \Cref{sec:appendix_thinning}, where we show that the effect of thinning regarding the desired statistical quantities is negligible.
\item In order to reduce memory consumption, in practice $\hat{x}$ and $\hat{t}$ are computed from the sample according to Welford's algorithm \cite{wel62}.
\end{itemize}
\end{remark}

\subsection{Belief Propagation}\label{sec:markov_random_fields}
Conversely to the sampling based methods introduced in the previous sections, Belief Propagation (BP) \cite{pea82,tapfre03} is a sampling free algorithm for computing marginals of a given discrete multivariate distribution. The method is designed for graphical models and, in particular, we consider the application to Markov random fields (MRF) \cite{lau1996}. Let $\mathcal{G} = (\mathcal{V},\mathcal{E})$ be a graph with nodes $\mathcal{V}$ and edges $\mathcal{E}$, where in the case of imaging each node $\nu \in \mathcal{V}$ corresponds to an image pixel, and edges $(\nu,\nu') \in \mathcal{E} \subseteq \mathcal{V}^2$ define a neighborhood system of pixels. In a MRF we are given a random vector $X$ corresponding to the graph $\mathcal{G}$, consisting of discrete random variables $X_\nu$ for $\nu\in\mathcal{V}$ which take values in a label set $\mathcal{L} = \{l_1,...,l_{L}\}$. The Markov property is encoded in the model by requiring that for $I,J\subset\mathcal{V}$ not intersecting or adjacent,
$X_I \indep X_J|X_{\mathcal{V}\setminus (I\cup J)}$, that is, $X_I$ independent of $X_J$ conditioned on $X_{\mathcal{V}\setminus (I\cup J)}$.
Using BP, the marginal distributions
\[
    \Prob[X_\nu=x_\nu] = \sum_{\substack{x'\in \mathcal{L}^{|\mathcal{V}|}:\\x'_\nu=x_\nu}} \Prob[X=x'].
\]
can be computed efficiently with results being exact in case of tree-like graphs and highly accurate approximations otherwise. In turn the pixel-wise posterior expectation $\hat{x}$ and variance $\hat{t}$ are computed as
\[
\hat{x}_\nu = \sum_{x_\nu \in \mathcal{L}} x_\nu\Prob[X_\nu=x_\nu], \hspace{2em} \hat{t}_\nu = \sum_{x_\nu \in \mathcal{L}} (x_\nu - \hat{x}_\nu)^2\Prob[X_\nu=x_\nu],
\]
Due to its high accuracy, the result of the BP algorithm can be used as a benchmark for the sampling based methods.

\begin{figure}[htb]
\includegraphics[width=\linewidth]{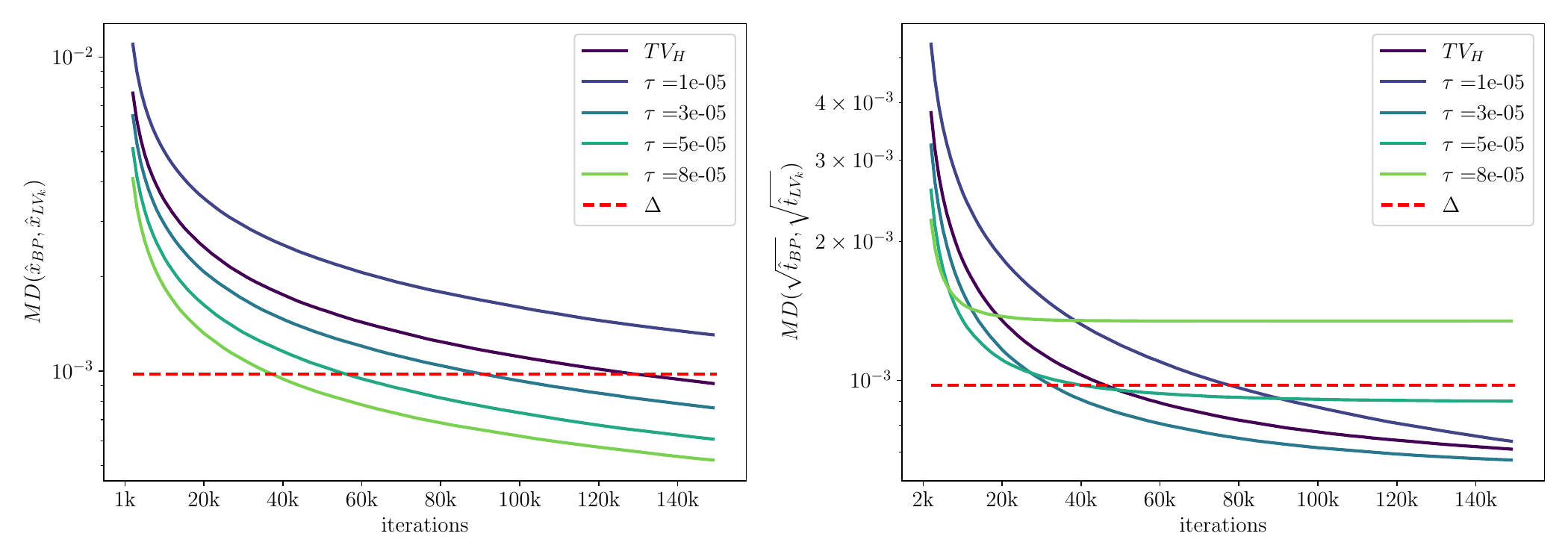}
\caption{Mean absolute difference of MMSE and posterior variance estimated using ULPDA for different values of $\tau$ and ULA with Huber TV ($\TV_h$) compared to BP as function of the number of samples with discretization threshold $\Delta=\frac{1}{1024}$.}
\label{fig:bp_comp_abs}
\end{figure}

\section{Numerical experiments}
\label{sec:num}
In this section, we present numerical results of various experiments with the proposed method, verifying our theoretical predictions. The source code to reproduce basic experiments can be found on \cite{error_estimation_source_code}. Throughout the experiments, we compute the error quantiles $\hat{s}_q$ of the reconstruction according to \Cref{algo:error_estimation} in different settings. In our first experiment, presented in \Cref{sec:tvl2_denoising}, we consider total variation based image denoising where we also present a comparison of the ULPDA results to the highly accurate BP results. Afterwards we replace the total variation regularization with a fields of experts \cite{ro09} regularizer in \Cref{sec:foe_denoising} and with total deep variation \cite{koef21} in \Cref{sec:tdv_denoising}. In \Cref{sec:mri_recon} we consider the inverse problem of accelerated magnetic resonance imaging showing the flexibility of our approach with respect to the forward operator of the inverse problem. In \Cref{sec:comp} we compare the proposed method to the one from \cite{angelopoulos2022image}. While the theory and the algorithms were presented for $X$ being one-dimensional, we consider high-dimensional image data in the following. This is done by applying our methods to each pixel separately, that is, the error quantile for an image is computed for each pixel separately based on the approximate posterior variance of this pixel's gray scale value. Doing so, each image sample provides as many samples of the relation between posterior variance and error as its number of pixels. However, strictly speaking, only one of these pixel-wise samples could be used in the quantile estimation, since different pixels of the same image are not independent in general. Our empirical investigation of this effect, see \Cref{sec:neighbouring_pixels}, however, has shown that using all pixels of a single image has a negligible effect on the obtained coverage in practice. For this reason, we use all pixels of each image sample in the subsequent experiments, and refer to \Cref{sec:neighbouring_pixels} for further details.

To evaluate our method quantitatively, we use two different metrics, the coverage and the magnitude of the estimated quantiles, where the former is a verification of the theoretical results as well as a measure for tightness of the quantiles and the latter a measure of the quality of the estimate. We define the coverage as the rate of correct predictions where the true reconstruction error is smaller or equal than the estimated quantile. Precisely, for an image with true error $s\in \RR^{M\times N}$ and predicted pixel-wise quantile $\hat{s}_q\in \RR^{M\times N}$, the coverage is defined as
\[
\text{coverage} = \frac{\left| \left\{(i,j)\mid s_{i,j}\leq(\hat{s}_q)_{i,j}\right\} \right|}{NM}.
\]
and for a data set containing multiple images we compute the mean of the coverage over all test images. The best method is the one yielding the smallest estimated quantiles satisfying the prescribed coverage. 

For our experiment we use gray scale images with values in $[0,1]$ from the BSDS training set corrupted with zero mean Gaussian noise with $\sigma=15/255$. For the estimation of the distribution of $(S,\hat{T})$ 400 images are used and we evaluate our method on a test set of 68 unseen images, referred to as BSDS 68 in the following.

\begin{figure}[htb]
\centering
\resizebox{0.8\textwidth}{!}{%
\includegraphics[scale=1]{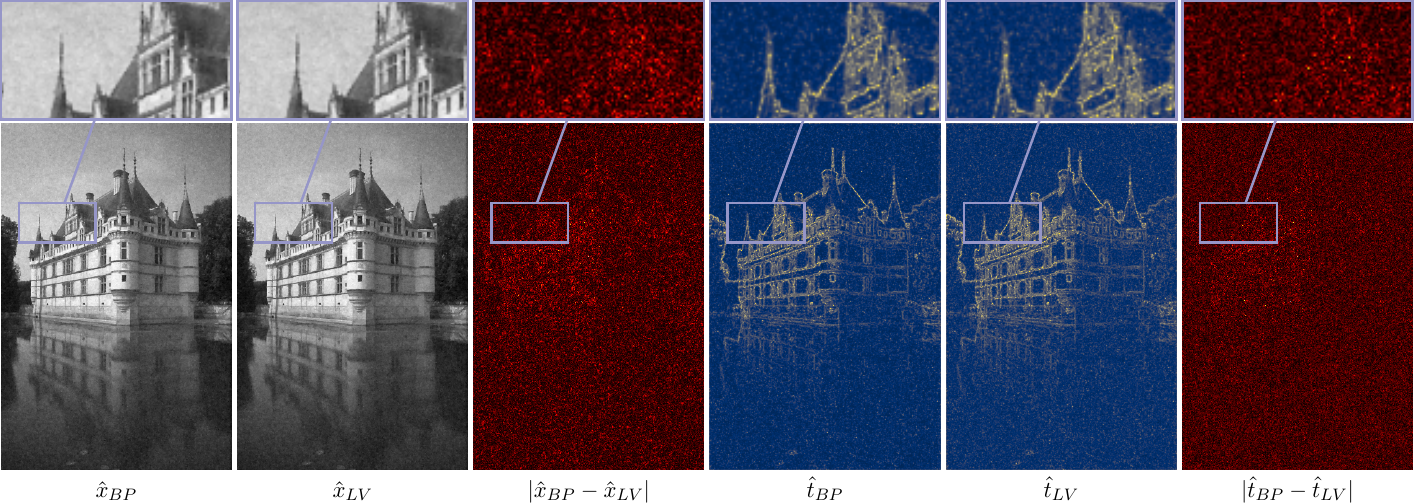}
}

\caption{Comparison of ULPDA-sampling results and BP reconstruction results for denoising with $\sigma=15/255$. The ULPDA reconstructions were obtained for a primal step size of $\tau=\expnumber{5}{-5}$ and 50k iterations ($\expnumber{8}{-4}$~\protect\includegraphics[width=1.5cm,height=.2cm]{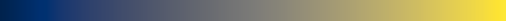}~$\expnumber{3.5}{-2}$, $0$~\protect\includegraphics[width=1.5cm,height=.2cm]{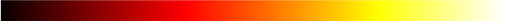}~$0.01/\expnumber{1}{-3}$ for MMSE and variance difference respectively).}
\label{fig:bp_comp_abs_vis}
\end{figure}

\subsection{TV-$\ell_2$ Denoising}\label{sec:tvl2_denoising}
In denoising the aim is to recover an unknown noise-free image $x \in \R^{M\times N}$ from a noisy observation $z\in \R^{M\times N}$, where the relationship between $x$ and $z$ is described by the following linear forward model
\[
    z = x + \nu,
\]
where each noisy pixel $\nu_{i,j}$ is independently sampled from a Gaussian distribution with zero mean and variance $\sigma^2$. Based on the forward model and the i.i.d. Gaussian assumption, the likelihood is simply given by the Gaussian distribution
\[
   \prob(z|x) \propto  \exp{\left( -\frac{1}{2\sigma^2} \norm[2]{x-z}^2\right)}.
\]
In this particular experiment, as in ~\cite{ch97,rof92},  we assume a prior~\footnote{Note that due to the lack of coercivity of $\TV(x)$ its associated prior is not well defined. However, by combining the prior with the likelihood function, the posterior is well-defined.}  based on the total variation,
\[
    \prob(x) \propto \exp{\left( -\frac{\TV(x)}{\lambda}\right)},
\]
where $\lambda$ plays the role of a variance parameter of the prior. Combining the likelihood with the prior, the negative log-posterior respectively the variational model is given by
\begin{equation}\label{eq:tvl2}
   -\log(\prob(x|z)) = \frac{1}{2\sigma^2} \norm[2]{x-z}^2 + \frac{1}{\lambda} \TV(x) + const,
\end{equation}
where we recall that an additive constant term is irrelevant for sampling via Langevin algorithms as only evaluations of the gradient of $-\log(\prob(x|z))$ are used. In what follows, we consider the anisotropic total variation
\[
\TV(x) = \sum_{i,j} |x_{i+1,j} - x_{i,j}| + |x_{i,j+1} - x_{i,j}| = \norm[1]{\Df x},
\]
where $\Df : \RR^{M\times N} \to
\RR^{M\times N\times 2}$ is a suitable finite differences operator, defined by
\begin{equation} \label{eq:defDx}
\begin{split}
  (\Df x)_{i,j,1} &= \begin{cases}
    x_{i+1,j} - x_{i,j} & \text{if }  1 \leq i < M,\\
    0 & \text{else},  \end{cases} \\*
  (\Df x)_{i,j,2} &= \begin{cases}
    x_{i,j+1} - x_{i,j} & \text{if }  1 \leq j < N,\\
    0 & \text{else}. 
  \end{cases}
\end{split}
\end{equation}
See for example~\cite{chpo16} for further information. The anisotropic total variation is chosen in order to render the problem amenable for the BP algorithm presented in \Cref{sec:markov_random_fields}. While we make use of the novel ULPDA, \Cref{algo:pd_langevin}, for the error estimation, in the following paragraph we first confirm functionality and estimate hyperparameters of this novel algorithm by a comparison to the highly accurate results obtained with BP.

\begin{figure}[htb]
\centering
\includegraphics[width=0.7\linewidth]{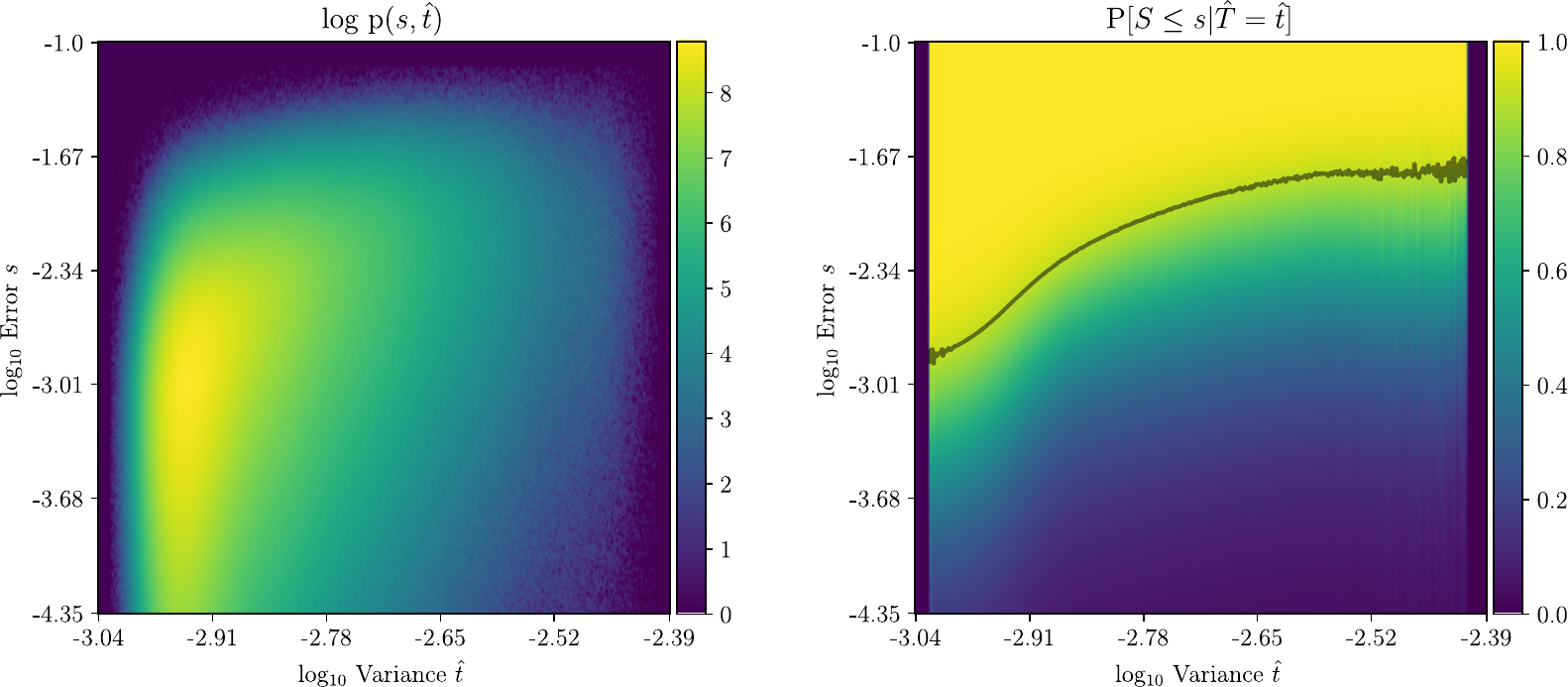}
\label{fig:empirical_dist_TV_l2}
\caption{TV-$\ell_2$ Denoising $\sigma = 15/255$. The left figure shows the joint log density of error S and estimated variance $\hat{T}$, while the right figure shows the conditional cumulative distribution for the error.
The black line indicates the conformalized 0.9 quantile.}
\end{figure}

\paragraph{A Comparison of ULPDA and Belief Propagation}
For BP to be applicable we have to discretize the TV denoising problem. Thus, we replace the continuous gray scale \emph{label} space $[0,1]$ by the discretization $l_k = k/L$, $k=0,...,L$. The node set $\mathcal{V}=\{(i,j)\;|\;i=1,\dots,M,\; j=1,\dots,N\}$ is the set of all image pixels and the edge set $\mathcal{E}$ the set of $4$-nearest neighbors on the pixel grid. In particular, in this setting the the log posterior probability factors nicely
\begin{equation}\label{eq:mrf_logdist}
    \begin{aligned}
        - \log \Prob(X=x|Z=z)= &- \log \Prob(Z=z|X=x)- \log \Prob(X=x) +const.\\
        =&\sum_{\nu \in \mathcal{V}} g_{\nu}(x_\nu)+\sum_{(\nu, \nu') \in \mathcal{E}} f_{\nu, \nu'}(x_\nu,x_\nu') + const,
    \end{aligned}
\end{equation}
where $g_\nu:\mathcal{L}\rightarrow \R$ are unary terms representing the data likelihood and the pairwise terms $f_{\nu, \nu'}:\mathcal{L}^2\rightarrow \R$ correspond to the prior. These are defined as
\[g_\nu(x_\nu) = \frac{1}{2\sigma^2}(x_\nu-z_\nu)^2,\quad f_{\nu, \nu'}(x_\nu, x_\nu') = \frac{1}{\lambda}|x_\nu-x_\nu'|.\]
In the following experiment we set the number of labels to $L=1024$ and use the Sweep Belief Propagation algorithm with 10 iterations to obtain the respective results.
In \Cref{fig:bp_comp_abs}, for a noisy observation $z$ with $\sigma=15/255$, on the left side we show the mean absolute difference $MD(\hat{x}_{BP},\hat{x}_{LV,k})=\frac{1}{MN} \sum_{i,j} |\hat{x}_{LV,k,ij}-\hat{x}_{BP,ij}|$ of the MMSE computed with BP ($\hat{x}_{BP}$) to the one computed with ULPDA ($\hat{x}_{LV,k}$) for different values of $\tau$ in dependence on the number of samples $K$ of ULPDA. On the right hand side we compare the posterior variances $\hat{t}_{BP}$ and $\hat{t}_{LV,k}$ in the same manner.
Additionally, we compare with the results of a smooth approximation of the total variation by Huber TV~\cite{chpo16} where we use the ULA scheme \Cref{algo:unaj_langevin} for sampling.
As \Cref{fig:bp_comp_abs} indicates, the UPDLA results converge faster for primal step sizes of $\tau >\expnumber{1}{-5}$.
Based in this results, for subsequent experiments with UPDLA, a primal step size $\tau=\expnumber{5}{-5}$ and $K = 50.000$ iterations is used since this yields a good tradeoff between approximation accuracy and computational cost. For the task of denoising, we further found that no burn in phase is necessary since a steady state is acquired within a few iterations.
The qualitative visual comparison between the BP and ULPDA results for the favoured hyperparameters, depicted in \Cref{fig:bp_comp_abs_vis}, show that no visual deviation can be observed between $\hat{x}_{LV}$ and $\hat{x}_{BP}$. 
The same holds for the estimated posterior variances $\hat{t}_{LV}$ and $\hat{t}_{BP}$. A visualization of the pixel marginals obtained with the two methods can be found in \Cref{fig:margs}.

\paragraph{Results for Error Estimation}
In \Cref{fig:empirical_dist_TV_l2} we show the resulting estimated distribution of $(S,\hat{T})$ obtained from posterior sampling with ULPDA. Exemplary reconstruction results alongside the pixelwise estimated quantile are shown in \Cref{fig:rec_1e-4}. The quantitative results for this experiment can be found in \Cref{tab:quantquant}.

\begin{figure}[htb]
\centering
\resizebox{0.8\textwidth}{!}{%
\includegraphics[scale=1]{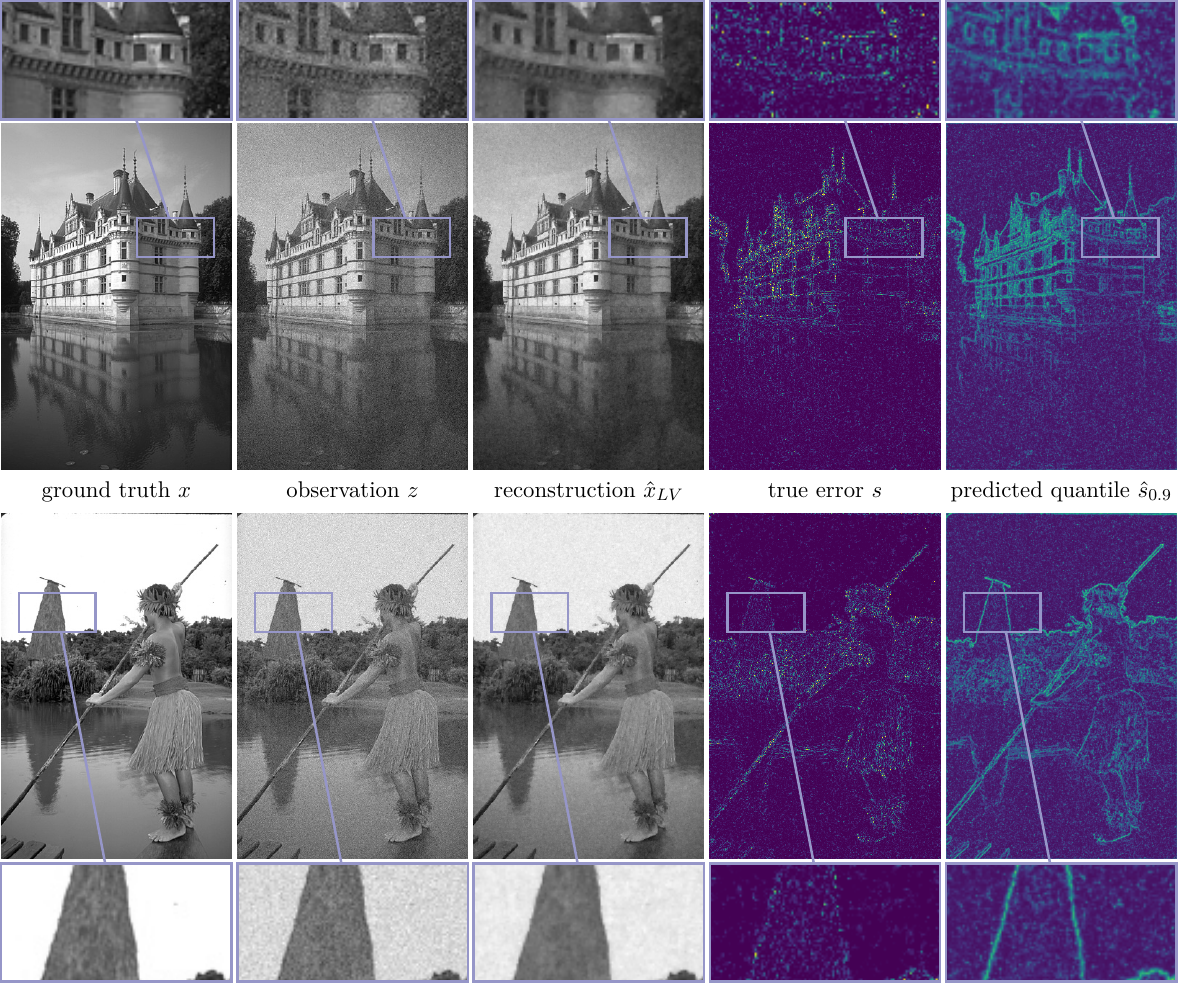}
}
\caption{TV-$\ell_2$ Denoising. Reconstruction results for samples from the BSDS dataset and a noise level of $\sigma = 15/255$.
From left to right: ground truth image~$x$, observation~$z$,  reconstruction~$\hat{x}_{LV}$, true error~$s$, and predicted $0.9$ quantile $\hat{s}_{0.9}$ ($0$ \protect\includegraphics[width=1.5cm,height=.2cm]{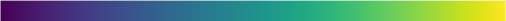} $0.02$).}
\label{fig:rec_1e-4}
\end{figure}

\subsection{Fields-of-Experts-$\ell_2$ Denoising}\label{sec:foe_denoising}
In this section we consider the task of denoising again, but replace the total variation regularization by a more complex, non-convex data driven prior, the so-called Fields of Experts (FoE) \cite{ro09} which reads as
$\prob(x)~\propto~\exp{\left( -\frac{FoE(x)}{\lambda}\right)}$ with
\[
    FoE(x) = \sum_{c=1}^C\sum_{i=1}^{M}\sum_{j=1}^{N} \phi_c((k_c * x)_{i,j}),
\]
$C \in \N$ experts of the form $\phi_c(t) =\alpha_c \log (1+t^2_i)$ with $\alpha_c$ and 2D convolution kernels $k_c$ being the trainable parameters of the method. The respective negative log-posterior reads as
\begin{equation}\label{eq:FoEl2}
   -\log(\prob(x|z)) = \frac{1}{2\sigma^2} \norm[2]{x-z}^2 + \frac{1}{\lambda} FoE(x) + const.
\end{equation}
Since the described prior is non-convex, the ULPDA algorithm is not applicable. However, sampling from the posterior is still possible using the ULA algorithm described \cref{algo:unaj_langevin} 

\begin{figure}[htb]
\centering
\includegraphics[width=0.7\linewidth]{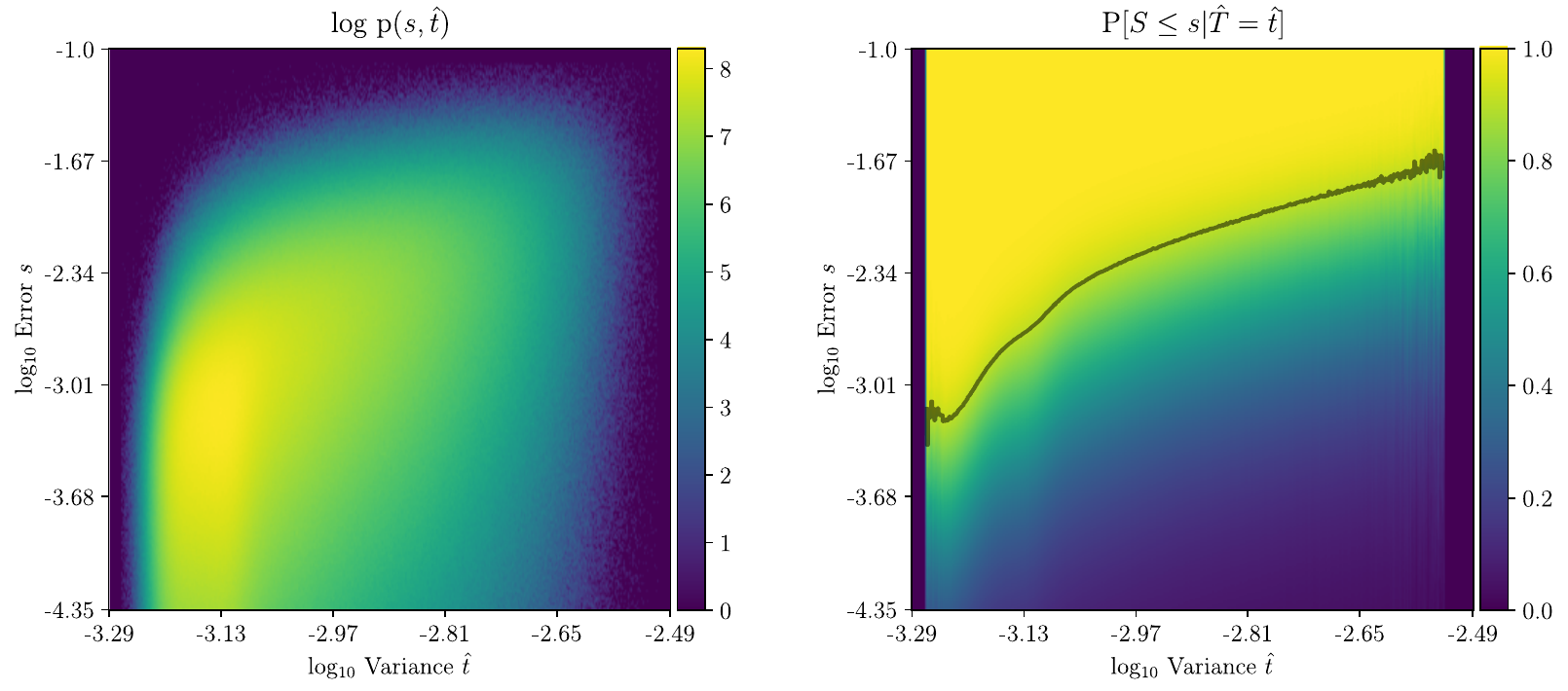}
\label{fig:foe_estimated_distribution}
\caption{FoE-$\ell_2$ Denoising $\sigma = 15/255$. The left figure shows the joint log density of error S and estimated variance $\hat{T}$, while the right figure shows the conditional cumulative distribution for the error.
The black line indicates the conformalized 0.9 quantile.}
\end{figure}

The FoE prior used for this experiment was trained on the BSDS dataset in a bi-level optimization scheme as presented in~\cite{ch14}.
The ULA algorithm is performed with 50k iterations. A step size of $\tau=\expnumber{1}{-4}$ and a regularization parameter of $\lambda = 0.125$ were determined empirically.
As before we use 400 images from the BSDS data set to estimate the distribution of $(S,\hat{T})$ and evaluate on a test set of 68 unseen images.
The estimated error-variance distribution can be found in \Cref{fig:foe_estimated_distribution}. Qualitative results are shown in \Cref{fig:rec_Foe_1e-4} and quantitative results in \Cref{tab:quantquant} again.

\begin{figure}[htb]
\centering
\resizebox{0.8\textwidth}{!}{%
\includegraphics[scale=1]{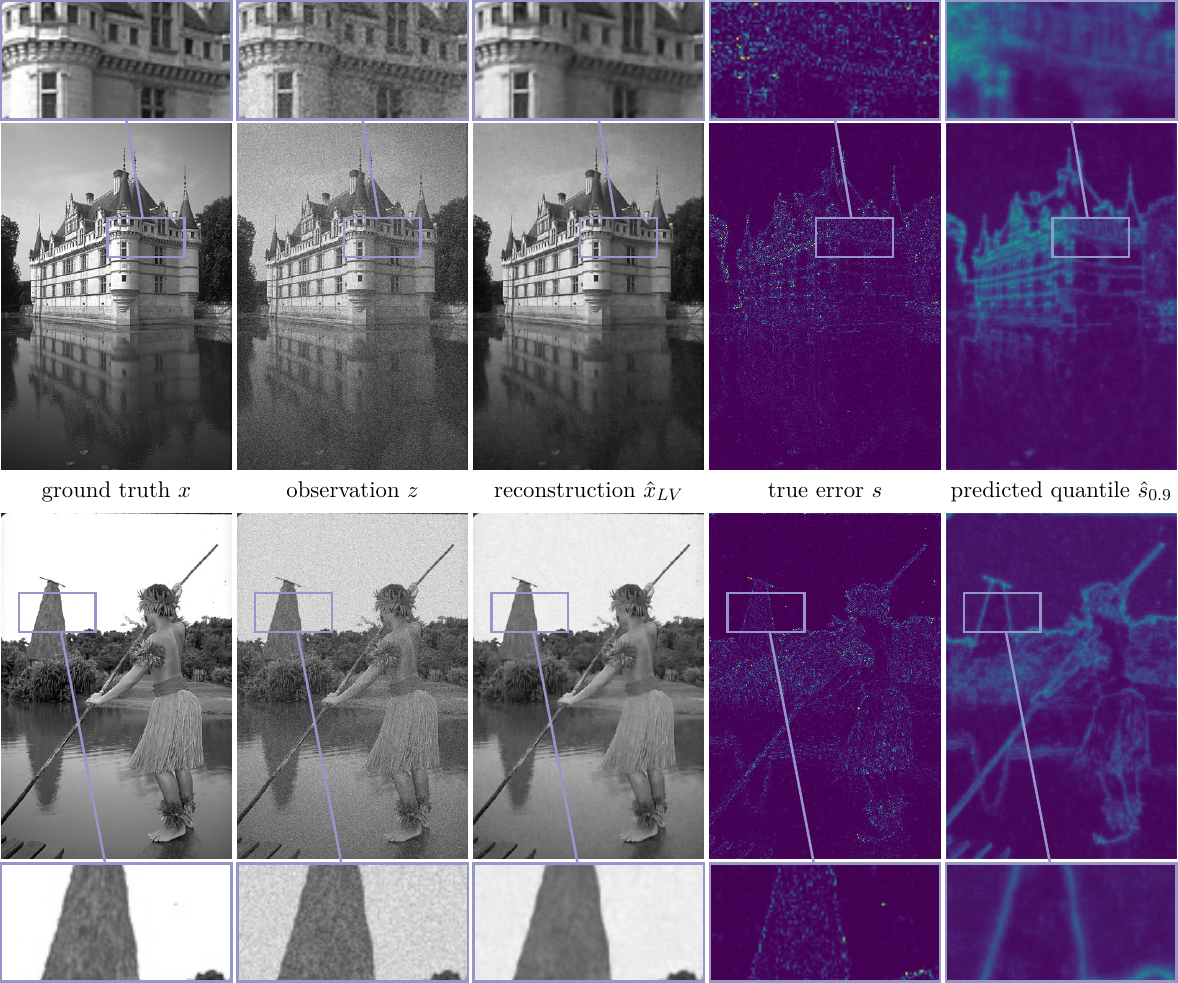}
}
\caption{FoE-$\ell_2$ Denoising. Reconstruction results for samples from the BSDS dataset and a noise level of $\sigma = 15/255$.
From left to right: ground truth image~$x$, corrupted image~$z$,  reconstruction~$\hat{x}_{LV}$ , true error~$s$, and predicted  $0.9$ quantile $\hat{s}_{0.9}$ ($0$ \protect\includegraphics[width=1.5cm,height=.2cm]{ures_viridis.png} $0.02$).}
\label{fig:rec_Foe_1e-4}
\end{figure}

\subsection{Total Deep Variation-$\ell_2$ Denoising}\label{sec:tdv_denoising}
The last experiment in the additive Gaussian denosing setting is performed using the Total Deep Variation (TDV)~\cite{koef21} regularizer.
As in the previous experiment, this regularizer is a non-convex data driven prior, however, orders of magnitude more powerful in terms of its approximation capability due to the greater amount of trainable parameters.
Formally, the prior is defined as $\prob(x) \propto \exp{\left( -\frac{TDV(x)}{\lambda}\right)}$ with
\[
    TDV(x) = \sum_{i=1}^{M}\sum_{j=1}^{N} \Psi(x)_{i,j},
\]
where $\Psi:\R^{M\times N}\to\R^{M\times N}$ is a U-Net inspired convolutional neural network. We refer the reader to~\cite{koef21} for a more detailed description of the TDV. As in the prevoius experiment, we use the ULA for sampling. Values of $\tau =\expnumber{1}{-3}$, $\lambda =\frac{1}{5.7}$ and $\beta = 150$ have empirically proven to yield satisfactory results.
Note that the natural choice of $\beta = \frac{1}{\sigma^2}$ was neglected, since $\beta$ was subject to optimization in the training of the TDV and must therefore be chosen accordingly.
As before, we use 50k samples of the respective Langevin algorithm. 
The data setup is the same as in the previous experiment. 
The estimated error-variance distribution is shown in \Cref{fig:TDV_estimated_distribution}, qualitative results in \Cref{fig:rec_TDV}, and quantitative ones again in \Cref{tab:quantquant}.

\begin{figure}[htb]
\centering
\includegraphics[width=0.7\linewidth]{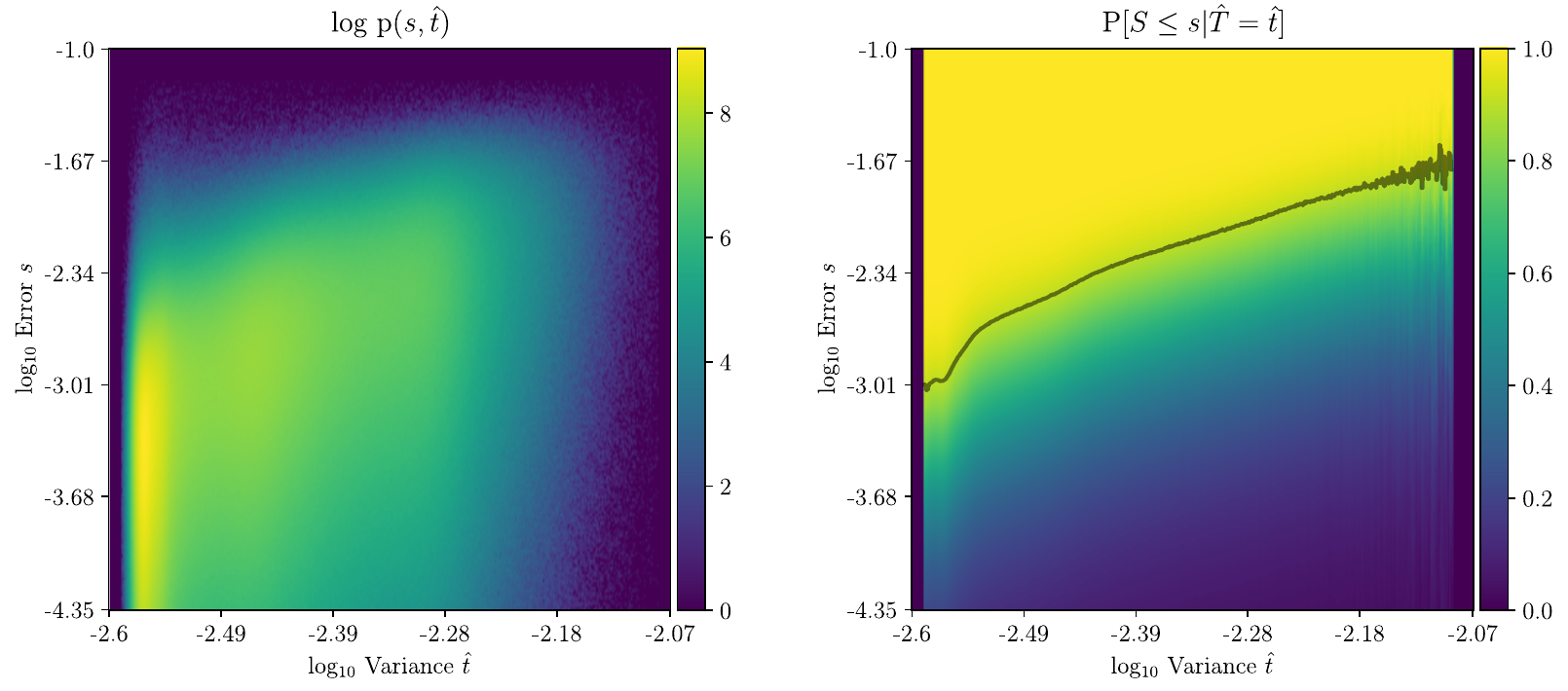}
\label{fig:TDV_estimated_distribution}
\caption{TDV-$\ell_2$ Denoising $\sigma = 15/255$. The left figure shows the joint log density of error S and estimated variance $\hat{T}$, while the right figure shows the conditional cumulative distribution for the error.
The black line indicates the conformalized 0.9 quantile.}
\end{figure}

\begin{figure}[htb]
\centering
\resizebox{0.8\textwidth}{!}{%
\includegraphics[scale=1]{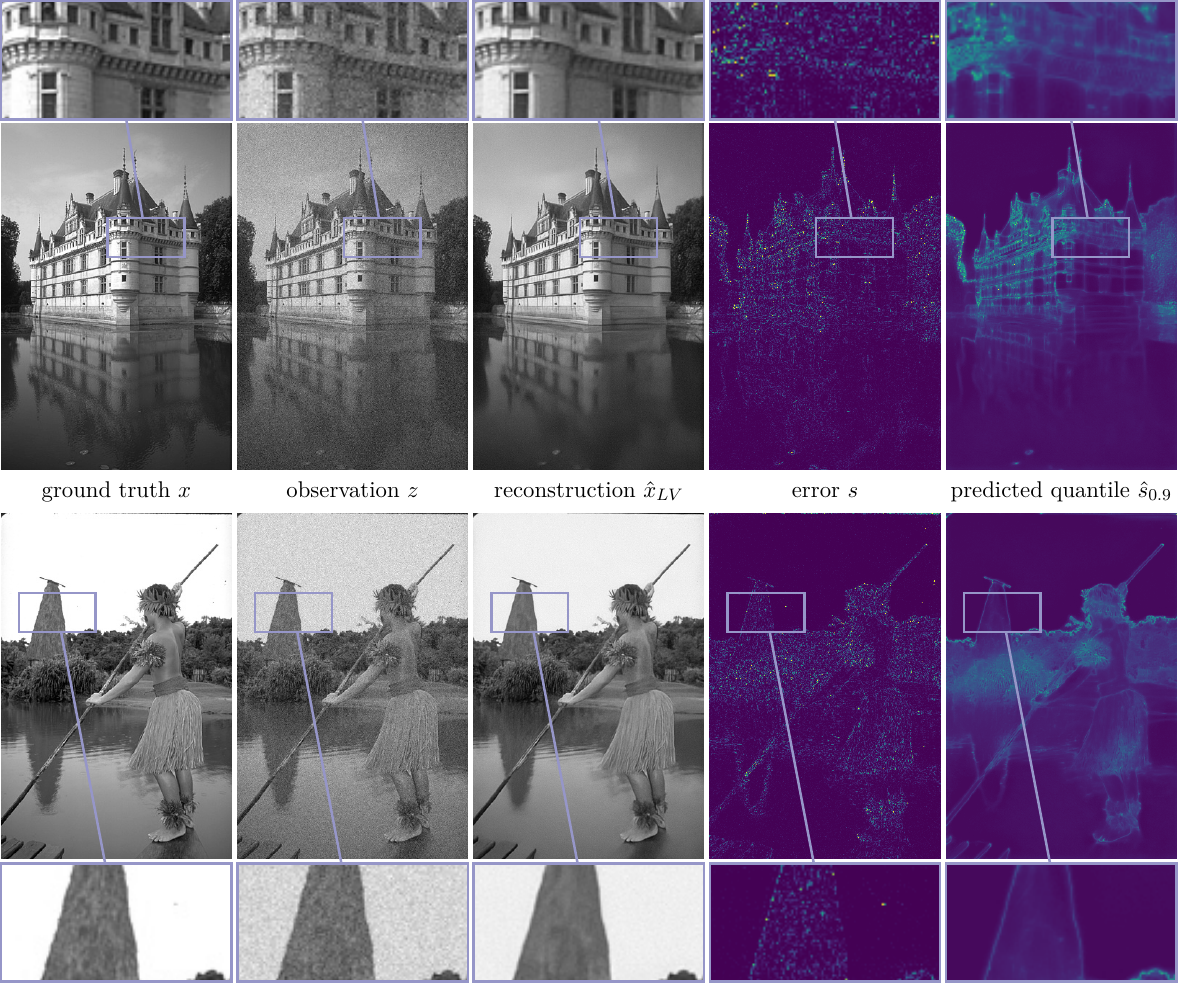}
}
\caption{TDV-$\ell_2$ Denoising. Reconstruction results for samples from the BSDS dataset and a noise level of $\sigma = 15/255$.
From left to right: ground truth image~$x$, corrupted image~$z$,  reconstruction~$\hat{x}_{LV}$ , true error~$s$, and predicted  $0.9$ quantile $\hat{s}_{0.9}$ ($0$ \protect\includegraphics[width=1.5cm,height=.2cm]{ures_viridis.png} $0.02$).}
\label{fig:rec_TDV}
\end{figure}

\subsection{MRI Reconstruction}\label{sec:mri_recon}

We now consider an inverse problem with a non-trivial forward operator and reuse the same TDV as in the previous section as a prior. The aim of accelerated MRI reconstruction is to recover an image $x \in \mathbb{R}^{M\times N}$ from a complex-valued observation $z \in \mathbb{C}^{M\times N}$ representing the undersampled k-space data. The relation of image and k-space data is given by the linear operator ${A:\mathbb{C}^{M\times N} \rightarrow\mathbb{C}^{M\times N}}$ composed of the two-dimensional discrete Fourier transform $\mathcal{F}:\C^{M\times N}\rightarrow \C^{M\times N}$ and an undersampling operator $M:\C^{M\times N}\rightarrow \C^{M\times N}$, together with additive measurement noise  $\nu$, formally expressed as
\begin{equation}
   z = Ax + \nu = M\mathcal{F}x + \nu.
   \label{eq:forward}
\end{equation}
Since the gradient step on the negative log-likelihood can be efficiently solved as proximal mapping, we use a proximal version of the ULA namely P-ULA~\cite{pe16} in this experiment. %
The proximal map of the negative log-likelihood reads as
\[
\prox_{\tau_k\mathcal{D}}(\Tilde{x})=\text{real}(\mathcal{F}^{-1}((\Id+\tau_k\beta M^\ast M)^{-1}(\mathcal{F}\Tilde{x}+\tau_k \beta M^\ast z))).
\]
Note that the TDV used in this setup is identical to the one used in the previous section and is therefore not specifically trained for the purpose of accelerated MRI reconstruction.
In contrast to the denoising experiments, the choice of the data-likelihood weighting is not clear, due to the linear operator.
As in \cite{narnhofer2021bayesian} we use a high value for the weighting parameter $\beta=\expnumber{1}{7}$ which almost leads to a projection on the acquired k-space lines due to the proximal map.
The regularization parameter was set to a value of $\lambda = \frac{1}{15}$.
Another difference to the denoising experiments is that we conduct a burn in phase of 500 iterations to acquire a steady state in this setup.
We used 420 central sclices from the fastmri knee multicoil training dataset \cite{zb18} with CORPD contrast to generate real valued groundtruth images from the root-sum-of-squares solution, which is further used to estimate the distribution of $(S, \hat{T})$.
Evaluation was performed on a test set of 100 images from the respective validation dataset. The estimated joint distribution of error and variance is shown in \Cref{fig:joint_distribution_mri}, qualitative results of the method in \Cref{fig:CORPD41}, and quantitative ones in \Cref{tab:quantquantMRI}.

\begin{figure}[htb]
\centering
\includegraphics[width=0.7\linewidth]{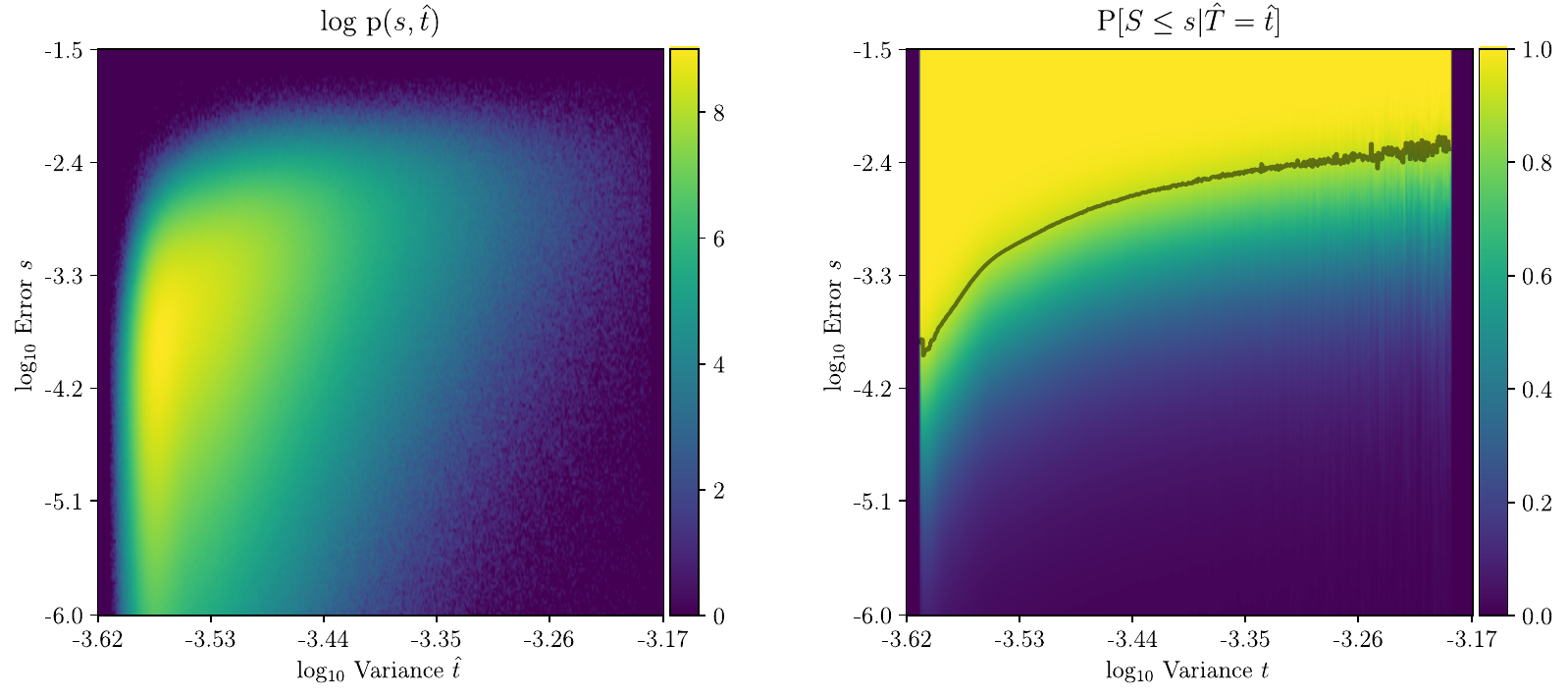}
\label{fig:joint_distribution_mri}
\caption{TDV-$\ell_2$ 4-fold MRI reconstruction. The left figure shows the joint log density of error S and estimated variance $\hat{T}$, while the right figure shows the conditional cumulative distribution for the error.
The black line indicates the conformalized 0.9 quantile.}
\end{figure}

\begin{figure}[htb]
\centering
\resizebox{0.8\textwidth}{!}{%
\includegraphics[scale=1]{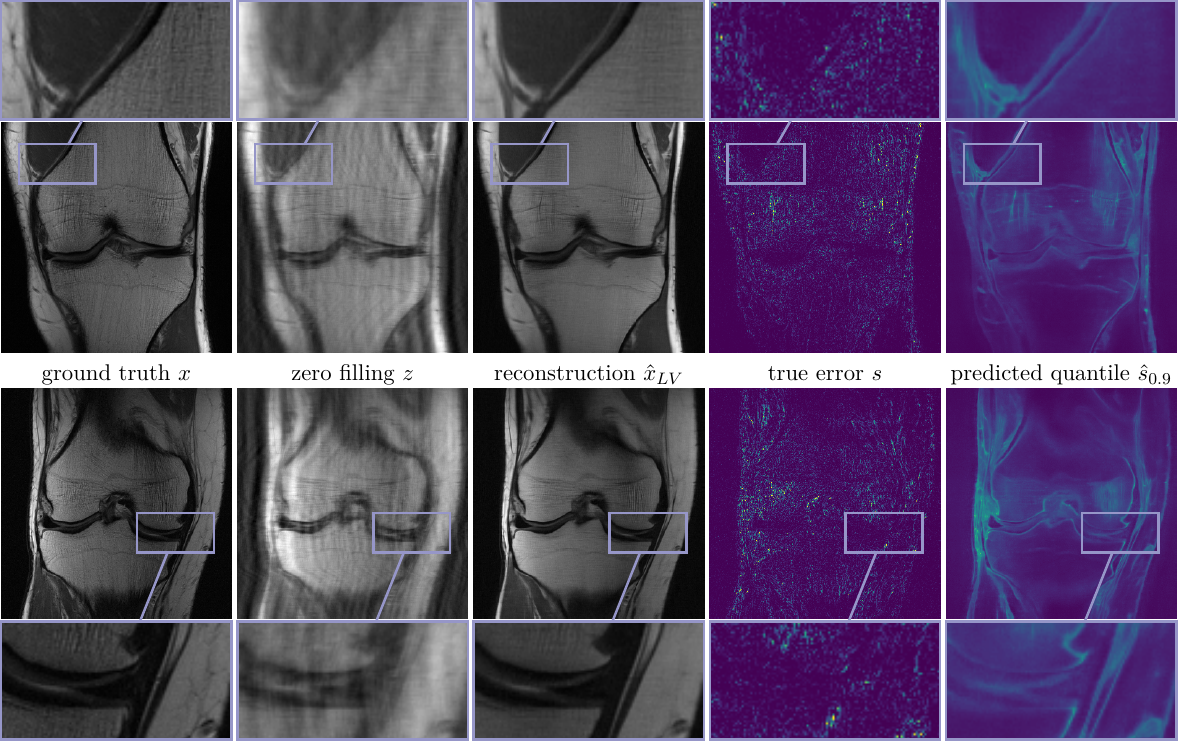}
}
\caption{MRI reconstruction results for CORPD data and $R=4$.
From left to right: ground truth image~$x$, zero filling~$z$,  reconstruction~$\hat{x}_{LV}$ , true error~$s$, and predicted  $0.9$ quantile $\hat{s}_{0.9}$~($0$~\protect\includegraphics[width=1.5cm,height=.2cm]{ures_viridis.png}~$0.06$)}
\label{fig:CORPD41}
\end{figure}

\subsection{Comparison to the State of the Art}\label{sec:comp}
In this section we compare the proposed method to an existing method for error estimation in imaging. While there are several methods available for uncertainty quantification, a fair comparison can only be made to other methods also estimating the prediction error.
\paragraph{The Comparison Method} Given a heuristic algorithm predicting pixel wise confidence intervals, in \cite{angelopoulos2022image} the authors propose a method of calibrating said algorithm in order to obtain guaranteed coverage. More precisely, assume we are given a heuristic predictor $\mathcal{T}:\Zc\rightarrow \R^{3\times N\times M}$, $\mathcal{T}(z)=(\hat{l}(z),\hat{x}(z),\hat{u}(z))$, where $\hat{x}$ is a point estimate of the ground truth image and $\hat{l}$, $\hat{u}$ denoting lower and upper bounds of the heuristic pixel-wise confidence intervals, satisfying $\hat{l}_{i,j}(z)<\hat{x}_{i,j}(z)<\hat{u}_{i,j}(z)$ for all pixels $(i,j)$. Define a scaled interval predictor for $\lambda>0$ as $\mathcal{T}_\lambda:\Zc\rightarrow \R^{3\times N\times M}$,
\begin{equation}
    \begin{aligned}
        &\mathcal{T}_\lambda(z)=(\hat{l}_\lambda(z),\hat{x}(z),\hat{u}_\lambda(z)),\\
        &\hat{l}_\lambda(z) = \hat{x}(z)+\lambda(\hat{l}(z)-\hat{x}(z)),\\
        &\hat{u}_\lambda(z) = \hat{x}(z)+\lambda(\hat{u}(z)-\hat{x}(z)).
    \end{aligned}
\end{equation}
Using a \emph{calibration} data set $(X_i,Z_i)_{i\in\Ic_\text{cal}}$, the parameter $\Lambda = \lambda((X_i,Z_i)_{i\in\Ic_\text{cal}})$ is chosen to ensure
\begin{equation}\label{eq:im2imuq}
\Prob\Bigg[\underbrace{\E\left[\frac{1}{NM}\left|\{(i,j)\;|\;X_{i,j}\in [l_\Lambda(Z)_{i,j},u_\Lambda(Z)_{i,j}]\}\right|\;\middle|\; (X_i,Z_i)_{i\in\Ic_\text{cal}}\right]}_\text{expected coverage conditioned on calibration data}\geq 1-\alpha\Bigg]\geq 1-\delta
\end{equation}
for $\alpha,\delta\in (0,1)$ defined arbitrarily by the user. That is, the method from \cite{angelopoulos2022image} ensures a bounded expected coverage on a new sample in probability over the calibration data. Two main differences to the proposed method are i) that in \cite{angelopoulos2022image} a method yielding a point estimate, a lower and an upper interval bound is needed already before the calibration, whereas for the proposed method a prior on the image data is sufficient, and ii) while the proposed method yields a direct probability on the coverage, in \eqref{eq:im2imuq} we obtain nesting of coverage and expectation. The benefit of the result in \eqref{eq:im2imuq} is that it allows for a prediction over the entire image, which comes at the cost of not knowing which pixels are the ones not correctly covered.
\paragraph{Experimental Setup} 
For the comparison we consider again TV-$\ell_2$ denoising as in \Cref{sec:tvl2_denoising}. We apply the method from \eqref{eq:im2imuq} with the heuristic interval predictor defined via 
\begin{equation}\label{eq:im2imheur}
\hat{l}(z) = \hat{x}(z)-\sqrt{\hat{t}(z)},\quad \hat{u}(z) = \hat{x}(z)+\sqrt{\hat{t}(z)}.
\end{equation}
with $\hat{t}(z)$ the estimated posterior variance. While in \eqref{eq:im2imuq} the authors present several different trained methods yielding heuristic interval predictors we decided to use \eqref{eq:im2imheur} since we think it allows for a reasonable comparison, focused on the calibration of the predictor ensuring coverage and rather than on the performance of the heuristic. Moreover, since the heuristic estimator would need to be trained on a distinct data set, additional sources of volatility with respect to the performance would be introduced via the training process. In order to obtain results comparable to our method, for the q-quantile we pick $\alpha$, $\delta$ performing a grid search minimizing the average estimated interval size while requiring that $(1-\alpha)(1-\delta)=q$ which is the most sensible choice according \eqref{eq:im2imuq} in order to obtain a coverage of approximately $q$ over all image pixels. The coverage results can be found in \Cref{tab:quant_comp} with a box plot comparing the interval sizes to the proposed method in \Cref{fig:boxcomp}. In this case the interval size of the proposed method is computed as $2\sqrt{\hat{s}_q}$, twice the estimated error, and the interval size of \cite{angelopoulos2022image} as $\hat{u}_\lambda-\hat{l}_\lambda$.

\begin{table}[htb]
\caption{Quantitative comparison of the proposed method and \cite{angelopoulos2022image} for BSDS 68 test data and the task of denoising with TV-$\ell_2$. Average coverage for different quantiles.}
\centering
\begin{tabular}{l|cc|cc|cc|cc}
{} &  \multicolumn{2}{c|}{Coverage}  &   \multicolumn{2}{c|}{Discrepancy to quantile in \% } \\
\noalign{\smallskip}
\textbf{Quantile} &{Reference} & {Ours} & {Reference} & {Ours}    \\  \hline\hline
     0.85          &  0.8973   &  0.8444  &   4.73  & 0.6\\
     0.90          &  0.9406   &  0.8951  &   4.06  & 0.49\\
     0.95          &  0.9805   &  0.9465  &   3.05  & 0.35\\\hline
\end{tabular}
\label{tab:quant_comp}
\end{table}

\begin{figure}[htb]
\label{fig:boxcomp}
\centering
\includegraphics[width=0.7\linewidth]{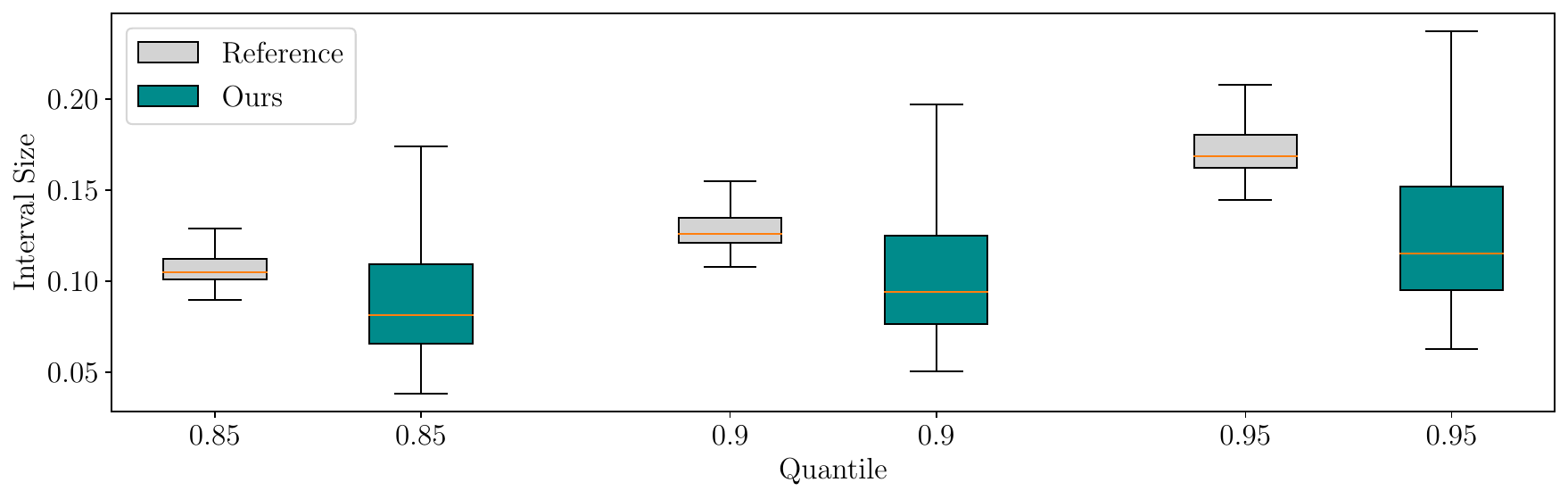}
\caption{Box plots of the estimated interval size on test data for different values of $q$. Comparison of the proposed method and the reference \cite{angelopoulos2022image}. Each individual box shows the respective estimated interval size, with the box representing the inter quantile range and the orange bar being the median.}
\end{figure}

\subsection{Evaluation of the Numerical Experiments}
In this section, we evaluate and discuss the results of the performed experiments. For this, we stress that the aim of the proposed work is not to obtain highly accurate reconstructions of the unknown ground truth, but rather to obtain tight bounds of the reconstruction error satisfying the desired coverage. First of all, note that quantitatively in all experiments coverage is obtained accurately, as can be seen in \Cref{tab:quantquant,tab:quantquantMRI}. The slight deviation from the predicted coverage despite the theoretical results only requiring a finite data set can be explained since, while our method only relies on a finite data set, the evaluation of probabilites, indeed, requires an infinite amount of data in order to be exact. Coverage, however, is guaranteed by the theoretical results for all methods anyway. Differences in performance should, thus, be measured rather in terms of the magnitude of the predicted error quantiles. In \Cref{fig:boxlog_denoising} we therefore show box plots of the estimated error quantiles for our denosing experiments. We can observe that all methods yield practically useful error estimates with magnitudes of the order $10^{-3}$ to $10^{-2}$ for images with gray scale values in $[0,1]$. Over all values of $q$, we can clearly note that the TDV prior yields the tightest estimates, followed by the FoE prior and lastly the TV prior. This trend is not surprising, as TDV is the most complex of the three models and TV clearly the simplest one. Let us focus now on a qualitative analysis of the results. In \Cref{fig:empirical_dist_TV_l2,fig:foe_estimated_distribution,fig:TDV_estimated_distribution,fig:joint_distribution_mri} we find that for all conducted experiments, the cumulative distribution of the reconstruction error conditioned on the posterior variance is mostly monotonically increasing, empirically supporting the predictive capability of the posterior variance for the error. Exceptions of monotonicity are observed only for extreme values of the variance in the form of oscillations and are most likely caused by a lack of data in these regions. Throughout all experiments, we observe that the estimated error quantiles structurally resemble the true reconstruction errors, see \Cref{fig:rec_1e-4,fig:rec_Foe_1e-4,fig:rec_TDV,fig:CORPD41}. In general, errors are concentrated at image edges which can be expected as most priors, especially TV type, penalize precisely image edges. The main differences between different priors are found in the heatmaps' sharpness around edges and overall smoothness and are a result of locality of the respective prior. The TV prior acts only on three neighbouring pixels resulting in sharp edges with the disadvantage of yielding noisy estimates with non-trivial error quantiles also in constant regions of the image, see \Cref{fig:rec_1e-4}. The FoE prior yields smoother error estimates where noise is only visible slightly in constant regions, \Cref{fig:rec_Foe_1e-4}. Again the best results are obtained with TDV, \Cref{fig:rec_TDV}, where error estimates are practically zero in constant regions of the image while still providing more detailed estimates than FoE (see the details of the castle) at the cost of some oversmoothing, see for instance the predicted error at the bush in the lower image. 
Furthermore, the reconstruction in the second row of \Cref{fig:rec_TDV} shows a pixel wit a high error which was not detected by our method, which can be explained by the fact that we only guarantee a 90$\%$ chance of a respective error being below the estimated quantile. 
The overall tendency of TDV yielding the best and FoE the second best results is also reflected in \Cref{tab:quantquant} in the respective values of the mutual information of $S$ and $\hat{T}$, which is a quantitative measure for the predictive capabilities of $\hat{T}$ wrt. $S$, see \Cref{sec:MI}. We observe a significant increase from TV to FoE and a further, smaller increase to TDV. For the sake of completeness we also added the PSNR and SSIM values in the tables. Regarding a comparison to the state of the art, our method outperforms the one proposed in \cite{angelopoulos2022image} in terms of interval sizes as shown in \Cref{fig:boxcomp}. Among all considered quantiles, the interval sizes with the proposed method tend to be statistically smaller with tighter coverage as shown in \Cref{tab:quant_comp}. It should, however, be mentioned that the method from \cite{angelopoulos2022image} might yield improved results using a more sophisticated heuristic interval predictor as explained in \Cref{sec:comp}. Moreover, coverage on a per pixel base cannot be prescribed for the method in \cite{angelopoulos2022image} and the applied choice of $\alpha,\delta$ from \eqref{eq:im2imuq} satisfying $(1-\alpha)(1-\delta)=q$ might not be optimal.

\begin{table}[htb]
\caption{Quantitative Results for different regularization approaches for denoising with $\sigma=15/255$ on the BSDS 68 dataset. Left: PSNR, SSIM, and mutual information (\Cref{sec:MI}) of $(S;T)$ for different experiments. Right: Coverage for different estimated quantiles and experiments.}
\centering
\begin{tabular}{l|c|c|c|c||c|c|c|c}
\noalign{\smallskip}
\hline
 \backslashbox[25mm]{\footnotesize{\textbf{Method}}}{\footnotesize{\textbf{Metric}}} & PSNR  & SSIM  & I(S;T)  & \backslashbox[25mm]{\footnotesize{\textbf{Method}}}{\footnotesize{\textbf{Quantile}}} &    0.85      &    0.9    &    0.95   & 0.99   \\  \hline\hline
 
     Corrupted    &   24.78     &    0.5820       &      -     & & & & & \\
     TV-$\ell_2$  &   29.34     &    0.7913       &    0.1231   & TV-$\ell_2$   &    0.8444  &    0.8951    &    0.9465 &    0.9888    \\
     FoE-$\ell_2$ &   30.29     &    0.8370       &    0.1609   & FoE-$\ell_2$  &    0.8447  &    0.8948    &    0.9459 &    0.9883    \\
     TDV-$\ell_2$ &   30.79     &    0.8484       &    0.1650   & TDV-$\ell_2$  &    0.8527  &    0.9023    &    0.9516 &    0.9906    \\\hline
\end{tabular}
\label{tab:quantquant}
\end{table}

\begin{figure}[htb]
\label{fig:boxlog_denoising}
\includegraphics[width=\linewidth]{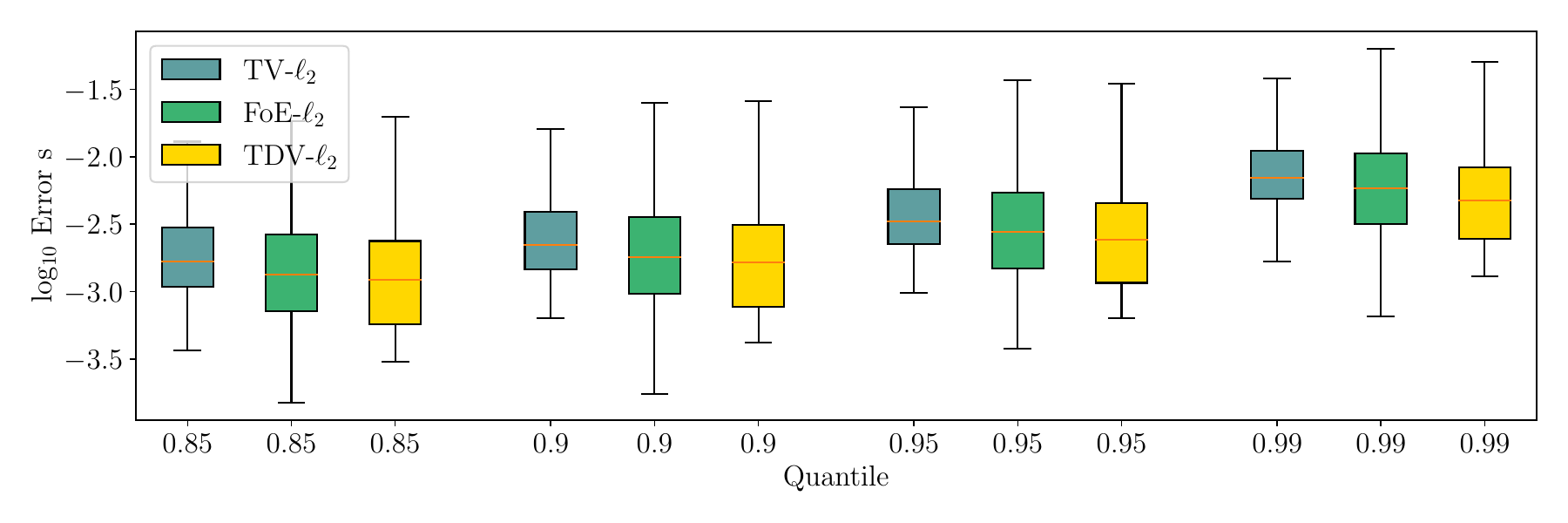}
\caption{Box plots of the estimated error quantiles for denoising on test data for different values of $q$ and different priors. Each individual box shows the respective distribution of errors, with the box representing the inter quantile range and the orange bar being the median.}
\end{figure}

\begin{table}[htb]
\caption{Quantitative Results for 4-fold accelerated MRI reconstruction with the TDV on the fastmri validation dataset.}
\centering
\begin{tabular}{l|c|c|c|c||c|c|c|c}
\noalign{\smallskip}
\hline
 \backslashbox[25mm]{\footnotesize{\textbf{Method}}}{\footnotesize{\textbf{Metric}}} & PSNR  & SSIM  & I(S;T) & \backslashbox[25mm]{\footnotesize{\textbf{Method}}}{\footnotesize{\textbf{Quantile}}} &    0.85      &    0.9    &    0.95   & 0.99 \\  \hline\hline 
     Corrupted    &   10.66     &    0.0145     &      -       &&&&&\\
     TDV-$\ell_2$ &   34.79     &    0.8812      &    0.1554    & TDV-$\ell_2$  &    0.8461  &    0.8966    &    0.9478 &    0.9895      \\\hline
\end{tabular}
\label{tab:quantquantMRI}
\end{table}

\section{Conclusions}
\label{sec:conclusions}

In this work we propose a general framework for error estimation for inverse problems in imaging that is flexible with respect to the forward operator as well as the chosen reconstruction approach as long as posterior sampling is possible. Given any point estimate for the inverse problem and an i.i.d. sample, the method enables to estimate pixel-wise quantiles of the reconstruction error for an unseen sample. Coverage of the estimated error quantiles is guaranteed without assumptions on the underlying distributions. In various experiments we show that the proposed method works accurately for different regularization approaches as well as inverse problems. Accurate knowledge about possible reconstruction errors might be especially important for critical tasks such as medical imaging where medical experts base the patients treatment on reconstructed images and therefore have an option to revise diagnoses based on structures that clearly show a high error potential. In future work the mutual information $I(S;T)$ between posterior variance and error might become an interesting metric to quantify the predictive capability of the posterior variance for the reconstruction error of a specific method. Furthermore future approaches could consider a different choice of statistical entities, such as substituting the MMSE by the median and the variance by the median absolute deviation and canonically the mean squared error by the mean absolute difference.

\paragraph{Limitations}
From a theoretical point of view, the presented results guarantee the desired coverage for each pixel individually. However, they do not directly carry over to an assertion regarding coverage of all pixels in an image. In order to circumvent this issue, in \cite{angelopoulos2022image} the authors bound the expected value of the rate of correctly covered pixels in an image using methods of risk control. A side effect of this approach, however, is that it is unknown which pixels are correctly covered and which are not. Requiring all pixels of an image to be covered correctly, on the other hand, was observed to lead to inefficient quantile estimates in our experiments, which is why we decided to use the pixel wise approach. Moreover, from a practical point of view the amount of data is a limiting factor. For numerical experiments we used all pixels of an image for estimation despite the fact that in this setting, the sample might not be i.i.d. While theoretically not favourable, this heuristic did not show to be disadvantageous in empirical experiments.

\appendix

\section{Analytical Computation of $\prob(s,t)$ for the 1D example}
\label{sec:appendixx_change_of_variables}
Let $X\in\Xc$ be a random vector with density $\prob_X$ and $Y=f(X)$ with a diffeomorphism $f:\Xc\rightarrow\Yc$. Then via a change of variables the density of $Y$ is computed as 
\[\prob_Y(y) = \prob_X(f^{-1}(y)) |\det(Df^{-1}(y))| \]
where $Df^{-1}$ denotes the Jacobian of $f^{-1}$. The computation can be generalized also in the case that $f$ is not injective. Assume $\Xc = \Xc_1\dot\bigcup \dots \dot\bigcup\Xc_n$ and $f$, $f(x)=g_i(x)$ if $x\in \Xc_i$ with $g_i:\Xc_i\rightarrow g(\Xc_i)\subset \Yc$ a diffeomorphism for all $i$. That is, $f$ might not be injective, but it is possible to partition its domain into subsets, such that $f$ is injective on each subset. In this case we find

\begin{equation}
    \begin{aligned}
        \Prob[f(X)\in A] = \int\limits_{f^{-1}(A)}\prob_X(x) \;\text{d}x \\
        = \sum\limits_{i=1}^n \int\limits_{f^{-1}(A)\cap X_i}\prob_X(x)\; \text{d}x \\
        = \sum\limits_{i=1}^n \int\limits_{A\cap g_i(\Xc_i)}\prob_X(g_i^{-1}(y))|\det(Dg_i^{-1}(y))| \;\text{d}y.
    \end{aligned}
\end{equation}
Thus, we find 
\[\prob_Y(y)=\sum\limits_{i:\;g_i^{-1}(\{y\})\neq \emptyset }\prob_X(g_i^{-1}(y))\;|\det(Dg_i^{-1}(y))|\]
In the setting at hand, the random vector $X$ is $(X,Z)$ and the transformation $f$ reads as 
\begin{align*}
\E[X|Z=z] &= \int x\prob(x|z)\;\text{d}x,\\
f_1(x,z) &= s(x,z) = (x-\E[X|Z=z])^2,\\
f_2(x,z) &= t(z) = \int (y-\E[X|Z=z])^2\prob(y|z)\;\text{d}y.\\
\end{align*}
The practical implementation of this procedure is illustrated in \Cref{fig:illustration_change_of:variables}. For given $(s',t')$, we first compute all values of $(x,z)$, such that $f(x,z)=(s',t')$, that is $f^{-1}(\{(s',t')\})$. In practice this is done by intersecting the corresponding level sets. For instance in the bottom, center plot in \Cref{fig:illustration_change_of:variables} the green lines denote the level set $L_s=\{(x,z)\;|\;s(x,z)=-1.75\}$, and the blue lines the level set $L_t=\{(x,z)\;|\;t(x,z)=-1\}$. Locally at every element $(x,z)\in L_s\cap L_t$, $f$ is invertible and we need the value of the determinant of its local inverse in order to compute $\prob_{(S,T)}(s',t')$. So let $(x',z')\in L_s\cap L_t$. The inverse $z(t)$ is difficult to compute and we thus linearly approximate it locally as
\begin{equation}
\hat{z}(t) = z_0+(t_1-t_0)\frac{z_1-z_0}{t_1-t_0}
\end{equation}
for $t\in [t_0,t_1]$ and $z_i = z_i(t_i)$. Afterwards, from $f_1(x,z) = s(x,z) = (x-\E[X\mid Z=z])^2$ we locally approximate $x(s,t)$ by
\begin{equation}
\hat{x}(s,t) = \E[X\mid Z=\hat{z}(t)] \pm \sqrt{s}
\end{equation}
with either $+\sqrt{s}$ if $x'>\E[X\mid Z=z']$ or $-\sqrt{s}$ otherwise. This immediately yields for the local inverse

\begin{align*}
\left| det\left( \begin{array}{cc} \frac{\partial x(s, t)}{\partial s} & \frac{\partial x(s, t)}{\partial t} \\ \frac{\partial z(s, t)}{\partial s} & \frac{\partial z(s, t)}{\partial t} \end{array} \right)\right| \approx \left|\frac{1}{2\sqrt{s}}\frac{z_1-z_0}{t_1-t_0}\right|.
\end{align*}
Finally, the density $\prob_{(S,T)}(s',t')$ is computed a
\[\prob_{(S,T)}(s',t')=\sum\limits_{(x',z')\in L_s\cap L_t}\prob_{(X,Z)}(x',z')\left|\frac{1}{2\sqrt{s'}}\frac{z_1-z_0}{t_1-t_0}\right|\]
with $z_i,t_i$ always chosen according to $z'$.

\begin{figure}[htb]
\label{fig:lines}
\includegraphics[width=\linewidth]{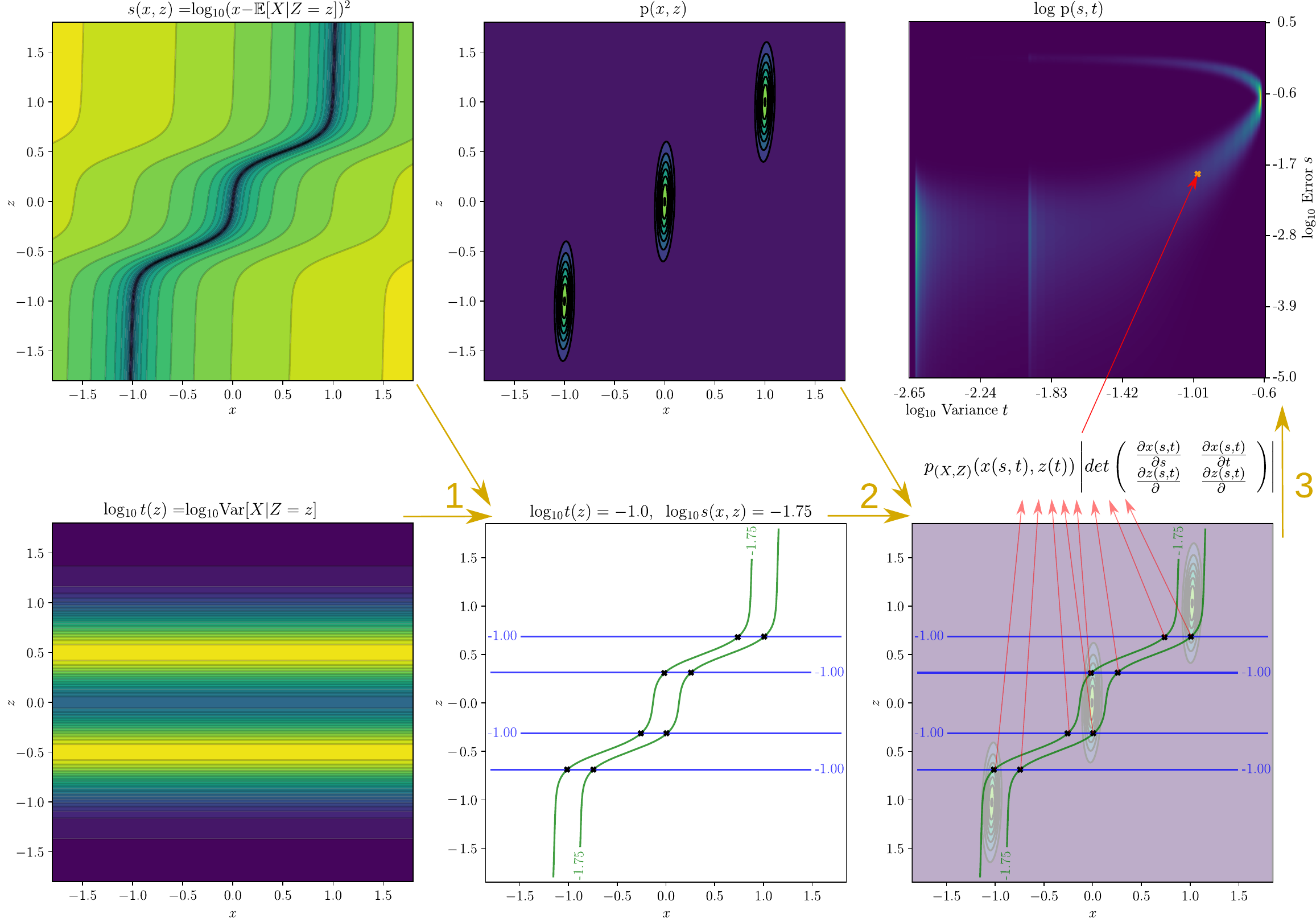}
\caption{Illustration of the methodology for analytically calculating the distribution $\prob(s,t)$. Step 1: Find level sets of $(s,t)$. Step 2: Find pre-image of $(s,t)$ and evaluate density. Step 3: Locally approximate inverse transformation, multiply the density with with the Jacobi determinant before summation.}
\label{fig:illustration_change_of:variables}
\end{figure}

\section{Statistical Independence of Neighbouring Pixels}\label{sec:neighbouring_pixels}
Estimating the distribution of a random variable from a finite sample in general requires the sample to be i.i.d. Thus, in order to compute the distribution of $(S,\hat{T})$ in theory we are not allowed to use the data of all pixels of an image, since i) neighbouring pixels are in general not independent of each other  and ii) if the prior or the forward operator acts globally, then two samples $X_1$, $X_2$ corresponding to two pixels of the same image will have the same observation $Z_1=Z_2$, namely the entire corrupted image. A remedy would be to consider each pixel position separately, that is, in order to estimate the reconstruction error of pixel $(i,j)$, we would only use exactly pixel $(i,j)$ of each image of the given sample. This approach, however, comprises a limitation for the approximation accuracy especially if the available amount of data is small. Fortunately, empirical evidence suggests that using all pixels of an image to compute a single joint density does not negatively impact the results. To show this we compare the results obtained from estimating the empirical distributions of $(S,\hat{T})$ and quantiles for each pixel position separately with the ones obtained from a joint estimation for TV-$\ell_2$ denoising. The setting at hand is identical to the experiment described in \Cref{sec:tvl2_denoising} except for the change of data which is required since this experiment needs to be performed with a vast amount of data.
We use $256\times256$ center crops of 35k Imagenet samples \cite{de09} for this experiment. Then, in \Cref{tab:quant_neighb_pixels}, we compare the obtained coverage from using the pixel-separate estimation with the one from using one empirical distribution estimated from all pixels jointly for our method. 
Discrepancies of the obtained coverage are of the order of hundredths of a percent and, thus, negligible. Therefore, we decided to use all pixels of each image for the estimation of the empirical distribution for our experiments. Qualitative examples of reconstructed Imagenet samples can be seen in \Cref{fig:ImagenetAux}.

\begin{table}[htb]
\caption{Quantitative comparison of pixel joint and separate estimation of the error-variance distribution for 256x256 Imagenet test data. Average coverage for different quantiles.}
\centering
\begin{tabular}{l|cc|cc}
{}&{Pixel separate} & {Pixel joint}  \\
\noalign{\smallskip}
 \textbf{Quantile} &  coverage &  coverage  &   discrepancy in \% \\  \hline\hline
     0.85          &  0.8512   &  0.8504  &   0.08\\
     0.90          &  0.9007   &  0.9029  &   0.22\\
     0.95          &  0.9499   &  0.9492  &   0.07\\
     0.99          &  0.9878	&  0.9875  &   0.03\\\hline
\end{tabular}
\label{tab:quant_neighb_pixels}
\end{table}

\section{Mutual Information of $s$ and $t$}\label{sec:MI}
The mutual information provides a means of quantifying the dependence of two random variables. More specifically, it quantifies how much information about one random variable is gained by observing the other one. The mutual information $\operatorname{I}$ of $S$ and $\hat{T}$ is defined as
\[
  \operatorname{I}(S;\hat{T}) = 
  \int \int
      {\prob(s,\hat{t}) \log{ \left(\frac{\prob(s,\hat{t})}{\prob(s)\,\prob(\hat{t})} \right) }
  } \; ds \,d\hat{t}.
\]

In our case, it can be interpreted as the reduction in uncertainty about the error $S$ observing the variance $T$ and can easily be computed by numerical integration using the estimated joint density $\prob(s,\hat{t})$.
Note that due to the mutual information being invariant under any smooth and uniquely invertible transformation of the variables~\cite{kr04}, we can directly compute it in log-space.

\section{Effect of Thinning}\label{sec:appendix_thinning}
Thinning is usually used in Markov Chain based sampling processes to eliminate the statistical dependence of consecutive samples. To be precise, approximating, e.g., the MMSE from a sample requires said sample to be i.i.d. Consecutive iterates of the Langevin algorithm are, however, not independent by nature. Given the output sequence $(x^k)_k$ of a Langevin algorithm, sufficiently far apart samples $x^{k}, x^{k+H}$ for large enough $H$ can be considered independent, which is why it is in general recommended to use $(x^{kH})_k$ for further computations.
This, however, requires an increase of anyway high computational costs caused by sampling in high dimensions. Thus, we investigate the actual empirical impact of thinning in the following, with the aim of justifying to use $H=1$ in our experiments. We compute the MMSE and posterior variance for different values of the thinning parameter $H$ fot TV-$\ell_2$ denoising. The mean absolute difference to the results obtained via the BP algorithm from \Cref{sec:markov_random_fields} over the number of samples used for the estimation can be seen in \Cref{fig:thin}, top row. One can observe that, indeed, any thinning accelerates convergence with respect to the number of samples. A given number of samples, however, requires the $H$-fold number of iterations of the Langevin algorithm if a thinning of $H$ is used. Hence, in the second and third row, we plot the deviations over the corresponding Langevin iterations and do not find a clear empirical advantage of thinning. For a thinning of $H=5$ or $H=10$, the curves are identical, for a thinning of $H=100$, convergence is even slowed down. For the sake of computation time, therefore, we decided to use $H=1$ in our experiments.

\begin{figure}[htb]
\label{fig:thin}
\includegraphics[width=\linewidth]{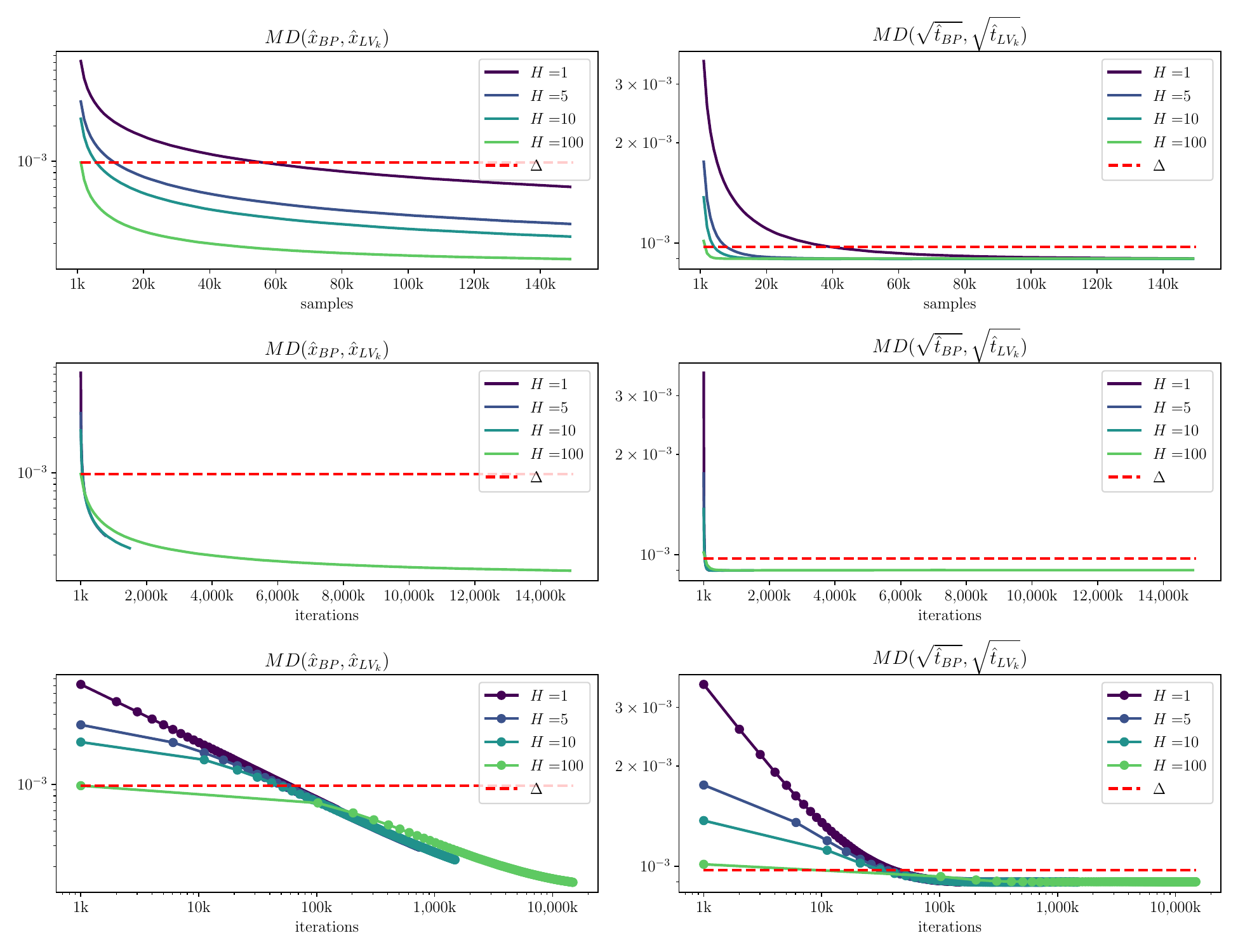}
\caption{The effect of thinning is shown over samples (first row) for deviation of MMSE $\hat{x}$ respectively posterior variance $\hat{t}$ from Langevin to BP. Bottom two rows show the effect over iterations respectively iterations in logarithmic representation. The discretization threshold is set to $\Delta=\frac{1}{1024}$.}
\end{figure}

\section{Additional Qualitative Plots for Denoising and MRI Reconstruction Experiments}\label{sec:appendix_qual_res}
In the following we show some more qualitative results obtained with the presented experiments. In \Cref{fig:margs} we show a qualitative comparison between the pixel marginals obtained with ULPDA sampling and the BP algorithm. In \Cref{fig:ImagenetAux} we show denoising results for Imagenet data from the experiment described in \Cref{sec:neighbouring_pixels}, \Cref{fig:appendix_qual_den2} shows additional results for denoising from BSDS data and in \Cref{fig:appendix_mri} additional MRI reconstruction results are depicted.

\begin{figure}[htb]
\centering
\includegraphics[width=\linewidth]{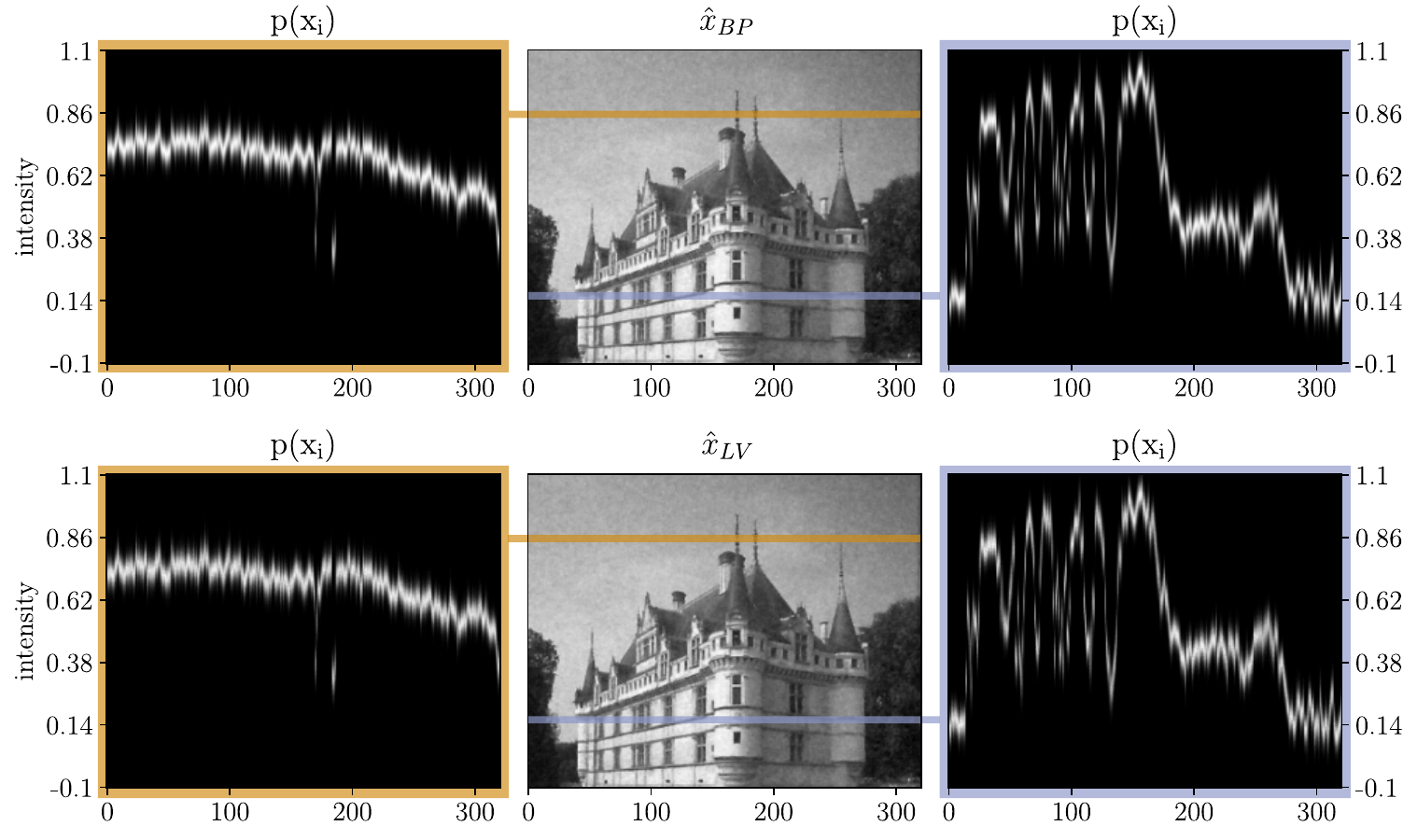}
\caption{Visualization of pixel marginals (left and right) for highlighted lines and the associated MMSE (middle). The upper row shows BP results, the lower row shows ULPDA results.}\label{fig:margs}
\end{figure}

\begin{figure}[htb]
\includegraphics[width=\linewidth]{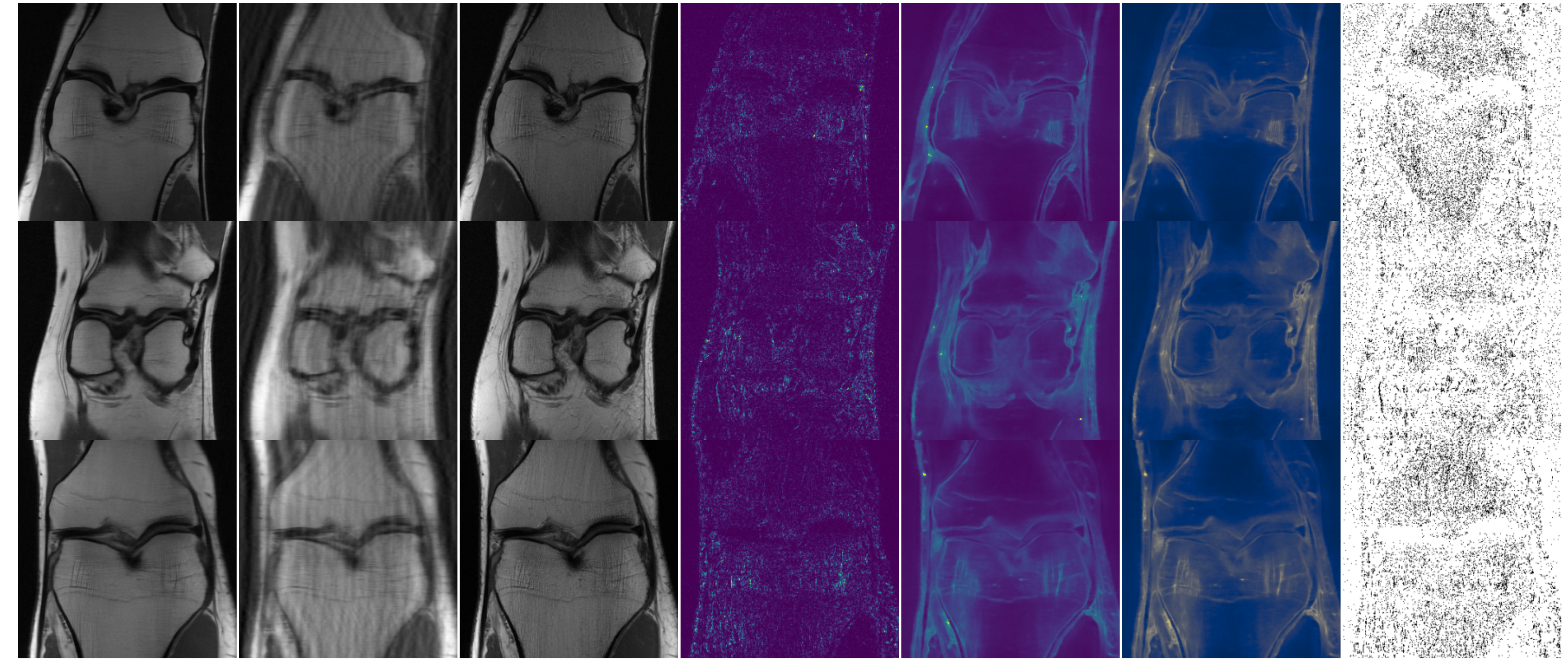}
\caption{MRI reconstruction with TDV prior. Various samples from fastmri dataset. Columns showing from left to right: reconstruction~$\hat{x}$, corrupted image~$z$, ground truth image~$y$, true error~$s$, predicted  $0.9$ quantile $\hat{s}_{0.9}$ ($0$ \protect\includegraphics[width=1.5cm,height=.2cm]{ures_viridis.png} $0.06$), estimated posterior variance $\hat{t}$ ($\expnumber{2.5}{-4}$~\protect\includegraphics[width=1.5cm,height=.2cm]{ures_cividis.png}~$\expnumber{5.6}{-4}$) and pixelwise succes (white if error below quantile, black if above).}
\label{fig:appendix_mri}
\end{figure}

\begin{figure}[htb]
\includegraphics[width=\linewidth]{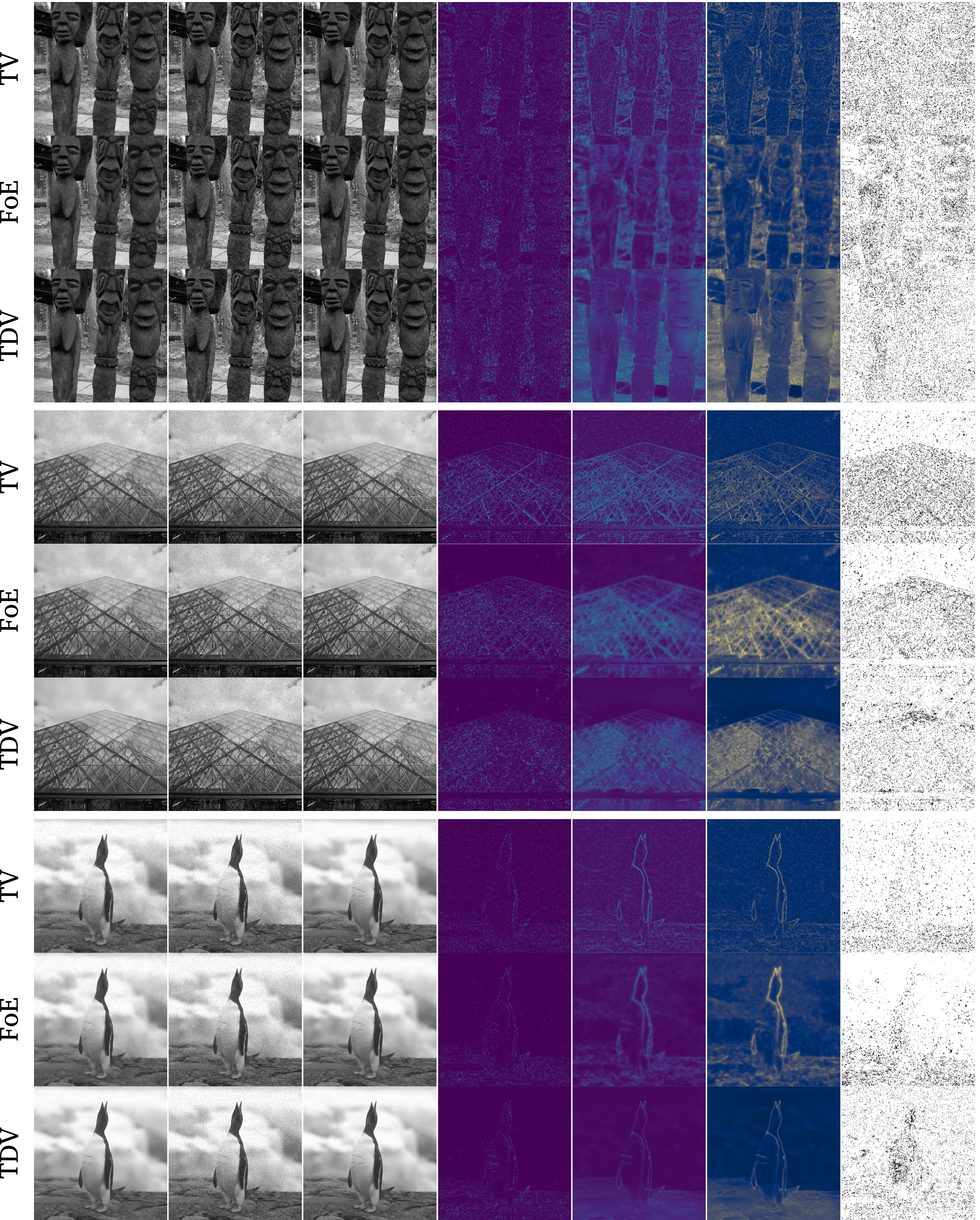}
\caption{Denoising. Various samples from BSDS68 for different priors. 
Columns showing from left to right: reconstruction~$\hat{x}$, corrupted image~$z$, ground truth image~$x$, true error~$s$, predicted  $0.9$ quantile $\hat{s}_{0.9}$ ($0$ \protect\includegraphics[width=1.5cm,height=.2cm]{ures_viridis.png} $0.03$), estimated posterior variance $\hat{t}$ and pixelwise succes (white if error below quantile, black if above).
Note that for the sake of contrast we omit colorbars for the variance plots in this figure.}
\label{fig:appendix_qual_den2}
\end{figure}

\begin{figure}[htb]
\includegraphics[width=\linewidth]{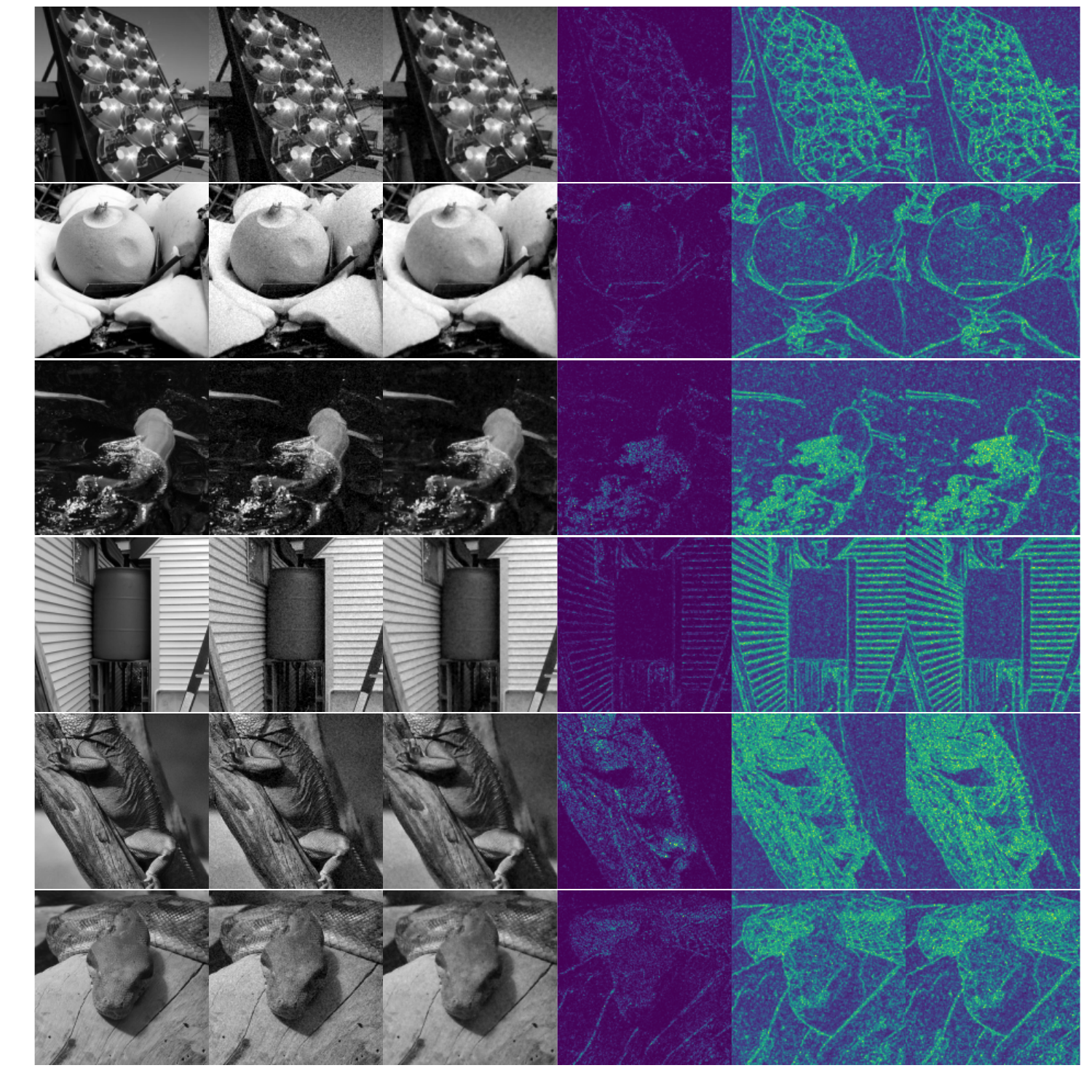}
\caption{TV-$\ell_2$ Denoising for various samples from Imagenet. From left to right: ground truth x, observation z, reconstruction  $\hat{x}$, true error s, predicted quantile $\hat{s}_{0.9}$ of the joint and pixelwise approach as described in \Cref{sec:neighbouring_pixels}. $0$ \protect\includegraphics[width=1.5cm,height=.2cm]{ures_viridis.png} $0.035$}
\label{fig:ImagenetAux}
\end{figure}

\printbibliography

\end{document}